\documentclass{article}

\usepackage[preprint]{neurips_2025}

\usepackage{microtype}
\usepackage{graphicx}
\usepackage{subfigure}
\usepackage{booktabs} 

\usepackage{hyperref}

\usepackage{algorithm}
\usepackage[noend]{algorithmic}

\usepackage{commath}

\usepackage{amsmath}
\usepackage{amssymb}
\usepackage{mathtools}
\usepackage{amsthm}
\usepackage{bbm}
\usepackage[capitalize,noabbrev]{cleveref}

\theoremstyle{plain}
\newtheorem{theorem}{Theorem}[section]
\newtheorem{proposition}[theorem]{Proposition}
\newtheorem{lemma}[theorem]{Lemma}

\theoremstyle{plain}

\newtheorem{assumption}[theorem]{Assumption}
\theoremstyle{plain}

\usepackage[textsize=tiny]{todonotes}

\usepackage{bbm}

\usepackage[utf8]{inputenc} 
\usepackage[T1]{fontenc}    
\usepackage{hyperref}       
\usepackage{url}            
\usepackage{booktabs}       
\usepackage{amsfonts}       
\usepackage{nicefrac}       
\usepackage{microtype}      
\usepackage{xcolor}         

\newcommand{\squishlisttwo}{
 \begin{list}{$\bullet$}
  { \setlength{\itemsep}{1pt}
     \setlength{\parsep}{0pt}
    \setlength{\topsep}{0pt}
    \setlength{\partopsep}{0pt}
    \setlength{\leftmargin}{1em}
    \setlength{\labelwidth}{1.5em}
    \setlength{\labelsep}{0.5em} } }
\newcommand{\squishend}{
  \end{list}  }

\title{Federated Linear Dueling Bandits}

\author{%
    Xuhan Huang$^{1}$, Yan Hu$^{1}$, Zhiyan Li$^{1}$, Zhiyong Wang$^{2}$, Benyou Wang$^{1}$, Zhongxiang Dai$^{1}$\thanks{Corresponding author. Email: daizhongxiang@cuhk.edu.cn} \\
    1 The Chinese University of Hong Kong, Shenzhen \\
    2 The Chinese University of Hong Kong
}

\begin{document}

\maketitle

\begin{abstract}
Contextual linear dueling bandits have recently garnered significant attention due to their widespread applications in important domains such as recommender systems and large language models. Classical dueling bandit algorithms are typically only applicable to a single agent. However, many applications of dueling bandits involve multiple agents who wish to collaborate for improved performance yet are unwilling to share their data. This motivates us to draw inspirations from \emph{federated learning}, which involves multiple agents aiming to collaboratively train their neural networks via gradient descent (GD) without sharing their raw data. Previous works have developed federated linear bandit algorithms which rely on closed-form updates of the bandit parameters (e.g., the linear function parameters) to achieve collaboration. However, in linear dueling bandits, the linear function parameters \emph{lack a closed-form expression} and their estimation requires minimizing a loss function. This renders these previous methods inapplicable. In this work, we overcome this challenge through an innovative and principled combination of online gradient descent (OGD, for minimizing the loss function to estimate the linear function parameters) and federated learning, hence introducing our \emph{federated linear dueling bandit with OGD} (\texttt{FLDB-OGD}) algorithm. Through rigorous theoretical analysis, we prove that \texttt{FLDB-OGD} enjoys a sub-linear upper bound on its cumulative regret and demonstrate a theoretical trade-off between regret and communication complexity. We conduct empirical experiments to demonstrate the effectiveness of \texttt{FLDB-OGD} and reveal valuable insights, such as the benefit of a larger number of agents, the regret-communication trade-off, among others.
\end{abstract}

\vspace{-2mm}
\section{Introduction}
\vspace{-2mm}
\emph{Contextual dueling bandits} \citep{NeurIPS21_saha2021optimal,ALT22_saha2022efficient,ICML22_bengs2022stochastic,arXiv24_li2024feelgood} have received significant attention recently, owing to their widespread adoption in important real-world applications such as recommender systems \citep{JCSS12_yue2012k} and large language models (LLMs) \citep{lin2024prompt,ji2024reinforcement}. 
In every iteration of a contextual dueling bandit problem, an agent receives a $d$-dimensional \emph{context} vector and $K$ \emph{arms}, selects a pair of arms, and collects a binary observation indicating the relative preference between the selected pair of arms \citep{ICML22_bengs2022stochastic}.
For example, when applying a contextual dueling bandit algorithm to optimize the response from an LLM, we receive a context (i.e., a prompt), use the LLM to generate $K$ responses (i.e., arms), select from them a pair of responses, and then ask the user which response is preferred \citep{lin2024prompt}.
In order to intelligently select the pairs of arms, we often adopt a surrogate function to model the \emph{latent reward function} governing the preference feedback observations (more details in Sec.~\ref{sec:background}), for which we typically adopt a linear function.
That is, we assume that the reward function value at arm $x$ can be expressed as $f(x) = \theta^{\top} \phi(x)$ for some unknown $\theta\in\mathbb{R}^d$ and known feature mapping $\phi(\cdot)\in\mathbb{R}^d$.
This is often referred to as the contextual \emph{linear dueling bandit} problem.

Classical contextual dueling bandit algorithms
are only applicable to single-agent scenarios.
However, many real-world applications involve multiple agents, which creates opportunities to enhance performance via \emph{collaboration among agents}.
Of note, in such applications, the agents are often concerned with privacy and are often \emph{unwilling to share their data}, including the selected arms and the observed preference feedback.
For example, users adopting contextual dueling bandits to optimize the response of LLMs may want to collaborate with each other without sharing their personal data involving their selected responses and preference feedback.
These requirements naturally align with the paradigm of \emph{federated learning} (FL), which enables multiple agents to collaboratively train their neural networks (NNs) without 
sharing their raw data \citep{mcmahan2017communication}.
In every round of FL, each agent calculates the local gradient (or parameters) of their NNs and sends them to the Central Server, who aggregates all received local gradients (typically via averaging) and then broadcasts the aggregated information back to the agents for further local updates of their NN parameters \citep{mcmahan2017communication}.

To extend contextual linear dueling bandits to the federated setting, it is natural to draw inspirations from previous works on federated contextual bandits \cite{shi2021federated}.
Notably, the work of \citep{wang2019distributed} has proposed a federated contextual linear bandit algorithm, which only requires the agents to share some sufficient statistics, including a $d$-dimensional vector and a $(d\times d)$-dimensional matrix.
Importantly, the two terms in the linear upper confidence bound (Lin-UCB) policy \cite{NIPS11_abbasi2011improved} for arm selection, including (a) the \emph{estimated linear function parameters} (i.e., the estimate of $\theta$) and (b) an exploration term, \emph{can be expressed in closed forms} in terms of the summation of these sufficient statistics.
As a result, in order to aggregate the information from all agents, the Central Server only needs to aggregate the sufficient statistics from all agents via a simple summation.
This allows every agent to effectively leverage the additional observations from the other agents to accelerate their bandit algorithm without requiring the exchange of raw observations.
This paradigm from \citep{wang2019distributed} has been widely adopted and extended to other bandit problems, such as federated neural bandits \citep{dai2022federated}.

However, in contextual linear dueling bandits, \emph{the estimated linear function parameters $\theta$ lack a closed-form expression} \citep{ICML22_bengs2022stochastic}.
Instead, we often need to estimate $\theta$ by minimizing a loss function via \emph{gradient descent} (GD).
This renders the federated bandit paradigm from \citep{wang2019distributed} inapplicable to our problem of contextual linear dueling bandits.
To resolve this challenge, we propose to allow the agents to \emph{collaboratively use GD to estimate the linear function parameters}.
Specifically, the joint loss function (to be minimized for parameter estimation) for all agents can be expressed in terms of a summation among all agents (more details in Sec.~\ref{sec:algo}). Therefore, the gradient of the joint loss function can be \emph{decoupled into the contributions from individual agents}. 
As a result, we let every agent calculate the local gradient of its own loss function and send it to the Central Server.
The Central Server then aggregates all local gradients (via summation) to attain the gradient of the joint loss function, which is then used for gradient descent.
Interestingly, in contrast to previous works on federated bandits \citep{shi2021federated,dai2022federated}, this novel approach bears a closer resemblance to the original federated learning (FL) paradigm which also involves the exchange of local gradients \citep{mcmahan2017communication}.


In this work, we firstly propose \textbf{a vanilla algorithm} named \emph{Federated Linear Dueling Bandit with Gradient Descent} (\texttt{FLDB-GD}), which adopts federated GD to estimate the linear function parameters $\theta$ in every iteration. 
A significant drawback of \texttt{FLDB-GD} is that it requires multiple rounds of gradient exchange between the Central Server and the agents in every iteration, making it impractical due to excessive communication costs.
Therefore, we also develop \textbf{a practical algorithm} named \emph{FLDB with Online GD} (\texttt{FLDB-OGD}), which only requires one round of gradient exchange after every $\tau\geq1$ iterations.
We perform rigorous theoretical analysis of the regret of our \texttt{FLDB-GD} and \texttt{FLDB-OGD} algorithms and show that they both enjoy sub-linear upper bounds on the cumulative regret (Sec.~\ref{sec:theory}).
Our theoretical results show that the vanilla \texttt{FLDB-GD} algorithm, which requires an excessive number of communication rounds, enjoys a tighter regret upper bound. 
Meanwhile, the practical \texttt{FLDB-OGD} algorithm has a worse regret upper bound yet is substantially more communication-efficient and hence more practical.
We also demonstrate a \emph{theoretical trade-off between the regret and communication complexity} of our \texttt{FLDB-OGD} algorithm.
Through extensive empirical experiments using synthetic and real-world functions, we show that both our algorithms consistently outperform single-agent linear dueling bandits (Sec.~\ref{sec:experiments}).
In addition, we also empirically demonstrate the benefit of having a larger number of agents, the robustness of our algorithms to the heterogeneity of the reward functions of different agents, and the practical trade-off between regret and communication for \texttt{FLDB-OGD}.


\vspace{-2mm}
\section{Background and Problem Setting}
\label{sec:background}
\vspace{-2mm}

At iteration $t$ of contextual dueling bandits,
the environment generates a set of $K$ arms, denoted as $\mathcal{X}_t\subset\mathcal{X} \subset \mathbb{R}^d$, where $\mathcal{X}$ represents the domain of all possible arms. The agent then selects a pair of arms $x_{t,1},x_{t,2} \in \mathcal{X}_t$ and receives feedback $y_t=\mathbbm{1}(x_{t,1} \succ x_{t,2})$, which is equal to $1$ if $x_{t,1}$ is preferred over $x_{t,2}$ and $0$ otherwise.

\textbf{Preference Model.}
Following the common practice in dueling bandits \citep{NeurIPS21_saha2021optimal, ICML22_bengs2022stochastic, arXiv24_li2024feelgood}, we assume that the preference feedback follows the Bradley-Terry-Luce (BTL) model \citep{AS04_hunter2004mm, Book_luce2005individual}.
Specifically, the utility of the arms are represented by a \emph{latent reward function} $f:\mathbb{R}^{d}\rightarrow \mathbb{R}$, which maps any arm $x$ to its corresponding reward value $f(x)$.
Here we assume that $f$ is a linear function $f(x) = \theta^{\top} \phi(x),\forall x$, in which $\theta\in\mathbb{R}^d$ are unknown parameters and $\phi(\cdot) \in\mathbb{R}^d$ denotes a known feature mapping. This reduces to standard linear dueling bandits when $\phi(\cdot)$ is the identity mapping.
Then, the probability that the first selected arm $x_{t,1}$ is preferred over the second selected arm $x_{t,2}$
is given by
\begin{align*}
\mathbb{P}\{x_{t,1} \succ x_{t,2}\} =\mathbb{P} \{ y_t = 1| x_{t,1}, x_{t,2}\}= \mu (f(x_{t,1}) - f(x_{t,2})). 
\end{align*}
Here 
$\mu(x) = 1/(1 + \mathrm{e}^{-x} )$ 
is the 
\emph{link function} for which we adopt the 
logistic function.

We list below the assumptions needed for our theoretical analysis, all of which are standard assumptions commonly adopted by previous works on dueling bandits
\citep{ICML22_bengs2022stochastic,ICML17_li2017provably}.
\begin{assumption}\label{assumption:basic}
\squishlisttwo
		\item $\kappa_\mu \triangleq \inf_{x,x' \in \mathcal{X}} \dot{\mu}(f(x) - f(x')) > 0$ for all pairs of arms~\footnote{Here $\dot{\mu}(\cdot)$ refers to the gradient of the link function.}.
		\item The link function $\mu : \mathbb{R} \rightarrow [0,1]$ is continuously differentiable and Lipschitz with constant $L_\mu$. For logistic function, $L_\mu \le 1/4$.
            \item The difference between feature maps is bounded $\norm{\phi(x_{t,1}) - \phi(x_{t,2})}_2 \leq 1$ for all arms. \\
\squishend
\vspace{-5mm}
\end{assumption}
\textbf{Performance Measure.} The goal of an agent in contextual dueling bandits is to minimize its \emph{regret}.
After selecting a pair of arms, denoted by $x_{t,1}$ and $x_{t,2}$, in iteration $t$, the learner incurs an instantaneous regret. In our theoretical analysis, we aim to analyze the following regret:
$
   r_t = 2f(x^*_t) - f(x_{t,1}) - f(x_{t,2})
$,
in which $x^*_t = \arg\max_{x\in \mathcal{X}_t} f(x)$ denotes the optimal arm in iteration $t$. 
After observing preference feedback for $T$ pairs of arms, the cumulative regret (or regret, in short) of a sequential policy is given by 
$
    R_T = \sum^T_{t=1} r_t = \sum^T_{t=1} \left(2 f(x^*_t) - f(x_{t,1}) - f(x_{t,2})\right)
$,
in which $x^*_t = \arg\max_{x\in \mathcal{X}_t} f(x)$ denotes the optimal arm in iteration $t$. 
A good policy should have sub-linear regret, i.e., $\lim_{T \to \infty}{R_T}/T = 0.$

\textbf{Federated Contextual Dueling Bandits.}
In federated contextual dueling bandits involving $N$ agents, we assume that all agents share the same latent reward function $f(x)=\theta^{\top} \phi(x)$.
This is consistent with the previous works on federated bandits \citep{shi2021federated,dai2022federated}.
In iteration $t$, every agent $i$ receives a separate set of arms $\mathcal{X}_{t,i} \subset \mathcal{X}$ and chooses from them a pair of arms denoted as $x_{t,1,i}$ and $x_{t,2,i}$.
In this case, we analyze the total cumulative regret from all $N$ agents:
$R_{T,N} = \sum^T_{t=1}\sum^N_{i=1} \left(2 f(x^*_{t,i}) - f(x_{t,1,i}) - f(x_{t,2,i})\right)$, 
in which $x^*_{t,i}={\arg\max}_{x\in\mathcal{X}_{t,i}}f(x)$.

Following the common practice in contextual bandits, we assume that the set of arms $\mathcal{X}_{t,i}$ are independently and identically sampled from the environment, and we make the following assumption on the distribution of the arm features $\mathcal{X}_{t,i}$.
\begin{assumption}\label{assumption:fed}
Let $\delta\in(0,1)$, define $\beta_t = \sqrt{2\log(1/\delta) + d\log\left( 1 + tN\kappa_{\mu}/(d\lambda) \right)}$.
For a fixed $\theta \in \mathbb{S}^{d}$, positive definite matrix $W \succeq \frac{\lambda}{\kappa_\mu} I_{d \times d}$, let 
$\widetilde{x}_{t,i,1} = \arg\max_{x\in\mathcal{X}_{t,i}}\theta^\top \phi(x)$, 
$\widetilde{x}_{t,i,2} = \arg\max_{x\in\mathcal{X}_{t,i}} \theta^\top \phi(x) + \frac{\beta_t}{\kappa_\mu}\lVert{\phi(x) - \phi(x_{1})}\rVert_{W^{-1}}$. 
Denote $\widetilde{\phi}_{t,i} = \phi(\widetilde{x}_{t,i,1}) - \phi(\widetilde{x}_{t,i,2})$, 
$\Sigma = \mathbb{E} \left[ \widetilde{\phi}_{t,i} {\widetilde{\phi}_{t,i}} ^T \right]$ 
and $\lambda_f = \displaystyle{\inf_{\theta, W, \beta}} \lambda_{\min} (\Sigma)$. We assume $\lambda_f>0$ is a positive constant.
\end{assumption}
Our Assumption \ref{assumption:fed} can be seen as a generalization of Assumption 3 of \citep{ding2021efficientalgorithmgeneralizedlinear} from generalized linear bandits to linear dueling bandits.
Such diversity assumptions are commonly adopted in the analysis of generalized linear bandits and dueling bandits \citep{li2017provablyoptimalalgorithmsgeneralized,wu2020stochasticlinearcontextualbandits,ding2021efficientalgorithmgeneralizedlinear}.
Intuitively, Assumption \ref{assumption:fed} implies that given the distribution from which the arm features are sampled and our arm selection strategy, the feature differences between the selected pairs of arms explore every direction of $\mathbb{R}^d$ with sufficient probability in every iteration to prevent blind spots in learning.
Specifically, this assumption guarantees that, on average, the feature differences supply adequate information in each iteration to facilitate sufficient exploration.

\vspace{-2mm}
\section{Federated Linear Dueling Bandits}
\label{sec:algo}
\vspace{-2mm}

In this section, we firstly introduce the vanilla algorithm: \texttt{FLDB-GD} (Sec.~\ref{subsec:algo:gd}), which results from an extension of linear dueling bandits to the federated setting and is impractical due to excessive communication costs. 
Next, we introduce our practical algorithm: \texttt{FLDB-OGD} (Sec.~\ref{subsec:algo:ogd}), which incurs significantly less communication costs while maintaining competitive performance.

\textbf{Local Arm Pair Selection.}
The local arm pair selection process is shared by both algorithms.
In iteration \( t \), each agent \( i \in [N] \) receives a set of \( K \) contexts, denoted as \( \mathcal{X}_{t,i} \) (line 3 of Algo.~\ref{algo:new:ogd:agent}), and selects two arms using the synchronized global parameters \( \theta_{\text{sync}} \) and \( W_{\text{sync}} \) (which contain information from all agents) following the classical linear dueling bandit (\texttt{LDB}) algorithm \citep{ICML22_bengs2022stochastic}. 
Specifically, the first arm $x_{t,1,i}$ is selected greedily based on \( \theta_{\text{sync}} \) (line 4 of Algo.~\ref{algo:new:ogd:agent}), which is the current estimated linear function parameters \emph{based on the data from all agents}.
Next, the second arm $x_{t,2,i}$ is chosen following an upper confidence bound strategy to balance exploration and exploitation (line 5 of Algo.~\ref{algo:new:ogd:agent}).
In this way, we encourage $x_{t,2,i}$ to both achieve high predicted reward and be different from $x_{t,1,i}$ as well as \emph{the previously selected arms from all agents}.
Then, the binary preference feedback between $x_{t,1,i}$ and $x_{t,2,i}$, denoted as $y_{t,i}$, is observed (line 6 of Algo.~\ref{algo:new:ogd:agent}).

Next, we discuss how \texttt{FLDB-GD} and \texttt{FLDB-OGD} facilitate the communication between the agents and the Central Server to obtain the synchronized global parameters \( \theta_{\text{sync}} \) and \( W_{\text{sync}} \).

\vspace{-1.5mm}
\subsection{The Vanilla Algorithm: \texttt{FLDB-GD}}
\label{subsec:algo:gd}
\vspace{-1.5mm}
To estimate the global parameters $\theta_{\text{sync}}$ in iteration $t$, the Central Server aims 
to find $\theta_{\text{sync}} = {\arg\min}_{\theta'} \mathcal{L}_t^{\text{fed}}(\theta')$, in which the loss function is given by:
\begin{equation}
\mathcal{L}_t^{\text{fed}}(\theta') = \sum^N_{i=1} \mathcal{L}_t^{i}(\theta') + \frac{1}{2}\lambda \lVert \theta'\rVert_2^2,
\label{eq:loss:func:fed}
\end{equation}
\begin{equation}
 \mathcal{L}_t^{i}(\theta') = - \sum^{t-1}_{s=1} \left( y_s^{i}\log \mu\left({\theta'}^{\top}\left[\phi(x_{s,1}^{i}) - \phi(x_{s,2}^{i})\right]\right) + (1-y_s^{i})\log \mu\left({\theta'}^{\top}\left[\phi(x_{s,2}^{i}) - \phi(x_{s,1}^{i})\right]\right) \right).
\label{eq:loss:func:fed:local}
\end{equation}
Equivalently, $\theta_{\text{sync}} = {\arg\min}_{\theta'} \mathcal{L}_t^{\text{fed}}(\theta')$ is the maximum likelihood estimate (MLE) of the unknown parameters $\theta$ \emph{given the data from all $N$ agents up to iteration $t-1$} \citep{ICML22_bengs2022stochastic}.

Since the loss function \eqref{eq:loss:func:fed} is convex for any $t$, it is natural to apply gradient descent (GD) to optimize~\eqref{eq:loss:func:fed}. 
Interestingly, the loss function \eqref{eq:loss:func:fed} is naturally decoupled across different agents, which allows the Central Server to \emph{estimate the gradient of \eqref{eq:loss:func:fed} using the gradients of the individual local loss functions \eqref{eq:loss:func:fed:local}}.
As a result, the agents only need to send their local gradients to the Central Server and are hence not required to share their local data $\{x_{t,1,i}, x_{t,2,i}, y_{t,i}\}$.
In addition, the Central Server updates the estimated $W_{\text{sync}}$ by aggregating all individual information matrices: $W_{\text{sync}} \leftarrow W_{\text{sync}} + \sum^N_{i=1} W_{\text{new},i}$.
For the reader's reference, we present this vanilla algorithm (named \texttt{FLDB-GD}) in Algos.~\ref {algo:new:gd:agent} and \ref{algo:new:gd:server} in App.~\ref{app:sec:fldb-gd}.

However, although the objective function~\eqref{eq:loss:func:fed} is convex, finding the optimal solution of the loss function at time $t$ requires multiple rounds of gradient descent. This \textbf{necessitates multiple rounds of gradient exchanges between the Central Server and the agents in every iteration}. This makes the vanilla algorithm \texttt{FLDB-GD} communication-inefficient and hence \textbf{impractical}.
To this end, in Sec.~\ref{subsec:algo:ogd}, we introduce our more practical algorithm \texttt{FLDB-OGD}, which uses online GD to estimate $\theta_{\text{sync}}$ and hence only needs a single round of communication after every $\tau\geq 1$ iterations.

\vspace{-1.5mm}
\subsection{The Practical Algorithm: \texttt{FLDB-OGD}}
\label{subsec:algo:ogd}
\vspace{-1.5mm}

\begin{algorithm}[h]
	\begin{algorithmic}[1]
        \STATE {\textbf{Initialization:}} $W_{\text sync} = \frac{\lambda}{\kappa_{\mu}}I_{d \times d}, \theta_{\text sync} = 0_{d \times 1}$
		\FOR{$t= 2, \ldots, T$}
            \STATE Receive contexts $\mathcal{X}_{t,i}$
            \STATE Choose the first arm  $x_{t,1,i} = \arg\max_{x\in\mathcal{X}_{t,i}}\theta_{\text{sync}}^\top \phi(x)$
            \STATE Choose the second arm $x_{t,2,i} = \arg\max_{x\in\mathcal{X}_{t,i}} \theta_{\text{sync}}^\top \left( \phi(x) - \phi(x_{t,1,i}) \right) + \frac{\beta_t}{\kappa_\mu}\lVert{\phi(x) - \phi(x_{t,1,i})}\rVert_{W_{\text{sync}}^{-1}}$
            \STATE Observe the preference feedback $y_{t,i} = \mathbbm{1}(x_{t,1,i}\succ x_{t,2,i})$ 
            \STATE Calculate local gradient $\nabla l_t^{i}(\widehat{\theta})$ \eqref{eq:loss:func:fed:local:ogd}
            \STATE Update $\nabla l_{\text{new}}^{i} \leftarrow \nabla l_{\text{new}}^{i} + \nabla l_t^{i}(\widehat{\theta})$
            \STATE Update $W_{\text{new},i} \leftarrow W_{\text{new},i} + \left(\phi(x_{t,1,i}) - \phi(x_{t,2,i})\right) \left(\phi(x_{t,1,i}) - \phi(x_{t,2,i})\right)^{\top}$
        \IF{$t \bmod \tau = 0$ (communication round starts)}
            \STATE Send $\{\nabla l_{\text{new}}^{i},W_{\text{new},i}\}$ to Central Server
            \STATE Receive $\{\theta_{\text{sync}}, W_{\text{sync}}, \widehat{\theta}\}$ from the Central Server
            \STATE $W_{\text{new},i} = \mathbf{0}$, $\nabla l_{\text{new}}^{i} = \mathbf{0}$
        \ENDIF
		\ENDFOR
	\end{algorithmic}
\caption{\texttt{FLDB-OGD} (Agent $i$)}
\label{algo:new:ogd:agent}
\end{algorithm}

\begin{algorithm}[h]
	\begin{algorithmic}[1]


            \STATE $t_c = t / \tau$
            \STATE Receive $\{\nabla l_{\text{new}}^{i},W_{\text{new},i}\}_{i=1,\ldots,N}$ from all $N$ agents
            \STATE Aggregate local gradients: $\nabla f_{t_c}^{\text{fed}}(\widehat{\theta}^{(t_c)}) = \sum^N_{i=1} \nabla l_{\text{new}}^{i}$ 
            \STATE Update parameters: $\eta_{t_c} = \frac{1}{\alpha t_c}$, $\widehat{\theta}^{(t_c+1)} =  \pi_S \left( \widehat{\theta}^{(t_c)} - \eta_{t_c} \nabla f_{t_c}^{\text{fed}}({ \widehat{\theta}}^{(t_c)}) \right)$, let $\widehat{\theta}=\widehat{\theta}^{(t_c+1)}$
            \STATE Update parameters: $\widetilde{\theta}^{(t_c+1)} = \frac{1}{t_c+1} \sum_{j = 1} ^{t_c+1} \widehat{\theta}^{(j)}$, let $\theta_{\text{sync}} = \widetilde{\theta}^{(t_c+1)}$
            \STATE Update $W_{\text{sync}} \leftarrow W_{\text{sync}} + \sum^N_{i=1} W_{\text{new},i}$
            \STATE Broadcast $\{\theta_{\text{sync}}, W_{\text{sync}}, \widehat{\theta}\}$ to all agents
	\end{algorithmic}
\caption{\texttt{FLDB-OGD} (Central Server)}
\label{algo:new:ogd:server}
\end{algorithm}

\texttt{FLDB-OGD} updates its estimate of $\theta_{\text{sync}}$ after every $\tau\geq1$ iterations (i.e., when $t \bmod \tau = 0$).
To estimate $\theta_{\text{sync}}$, the loss function our practical \texttt{FLDB-OGD} algorithm aims to minimize is the same as that of \texttt{FLDB-GD}~\eqref{eq:loss:func:fed},
but reformulated into the following form:
\begin{equation}
\mathcal{L}_t^{\text{fed}}(\theta') = \sum^t_{s=1} f^{\text{fed}}_s(\theta'),
\label{eq:loss:func:fed:ogd}
\end{equation}
where 
\begin{equation}
f^{\text{fed}}_s(\theta') =
\begin{cases} 
\displaystyle \sum_{i=1}^{N} l_s^{i}(\theta'), & \quad \text{if } s \neq 1, \\[15pt]
\displaystyle  \sum_{i=1}^{N} l_1^{i}(\theta') + \frac{\lambda}{2} \| \theta' \|_2^2, & \quad \text{if } s = 1,
\end{cases}
\label{eq:loss:func:fed:local:sum:ogd}
\end{equation}
\begin{equation}
    l_s^{i}(\theta') = - \Big( y_s^{i}\log\mu\big({\theta'}^{\top} \big[\phi(x_{s,1}^{i}) - \phi(x_{s,2}^{i}) \big]\big) + (1-y_s^{i})\log\mu\big({\theta'}^{\top}\big[\phi(x_{s,2}^{i}) - \phi(x_{s,1})\big]\big) \Big)
\label{eq:loss:func:fed:local:ogd}
\end{equation}
Intuitively, every individual loss function $f^{\text{fed}}_s$ \eqref{eq:loss:func:fed:local:sum:ogd} corresponds to the total loss of \emph{the data collected in iteration $s$} from all $N$ agents.
Inspired by the work of~\citep{ding2021efficientalgorithmgeneralizedlinear} which used OGD to estimate the parameters in generalized linear bandits, in a communication round $t_c=t/\tau$, we only use the gradients of $f^{\text{fed}}_s$ \eqref{eq:loss:func:fed:local:sum:ogd} from the newly collected data since the last communication round $t_c-1$
to update our estimation of $\theta_{\text{sync}}$.
That is, instead of GD, we use OGD to estimate $\theta_{\text{sync}}$.

At the beginning of our \texttt{FLDB-OGD} ($t=1$), the Central Server finds the minimizer of $\mathcal{L}_t^{\text{fed}}(\theta')$~\eqref{eq:loss:func:fed:ogd} at iteration $t=1$ 
(which is also the minimizer of $f^{\text{fed}}_1(\theta')$) denoted as $\widehat{\theta}^{(1)}$, and let $\widetilde{\theta}^{(1)} = \widehat{\theta}^{(1)}$.
This ensures that the subsequent estimated $\theta_{\text{sync}}$ always lies in a bounded ball centered at the groundtruth parameter $\theta$, which is needed in our theoretical analysis (Sec.~\ref{sec:theory}).
In the subsequent communication rounds $t_c=t/\tau$,
the server also maintains a ball centered at $\widehat{\theta}^{(t_c)}$ as the projection set during OGD: $\mathcal{S}=\left\{\theta:\left\|\theta-\widehat{\theta}^{(t_c)}\right\| \leq 2 r\right\}$, in which $r \triangleq \sqrt{ \left(2\log(1/\delta) + d\log\left( 1 + T N \kappa_{\mu}/(d\lambda) \right)\right) / (\lambda \kappa_{\mu}})$. 



In every iteration $t>1$, each agent $i$ keeps track of two parameters: $\nabla l_{\text{new}}^{i}$ (lines 7-8 of Algo.~\ref{algo:new:ogd:agent}) and $W_{\text{new},i}$ (line 9 of Algo.~\ref{algo:new:ogd:agent}), which represent the accumulated local gradients at 
$\widehat{\theta}$ 
and accumulated information matrix, respectively.
After every $\tau$ iterations, a communication round starts (line 10 of Algo.~\ref{algo:new:ogd:agent}), and every agent $i$ sends its local parameters $\nabla l_{\text{new}}^{i}$ and $W_{\text{new},i}$ to the Central Server (line 11 of Algo.~\ref{algo:new:ogd:agent}).
In communication round $t_c = t / \tau$, after the Central Server receives $\{\nabla l_{\text{new}}^{i},W_{\text{new},i}\}_{i=1,\ldots,N}$ from all $N$ agents, it aggregates all individual gradients $\nabla l_{\text{new}}^{i}$ via summation to obtain $\nabla f_{t_c}^{\text{fed}}(\widehat{\theta}^{(t_c)})$ following~\eqref{eq:loss:func:fed:local:sum:ogd} (line 3 of Algo.~\ref{algo:new:ogd:server}).
Then, the Central Server performs one step of projected gradient descent with step size \( \eta = \frac{1}{\alpha t_c} \) to obtain $\widehat{\theta}^{(t_c+1)}$ (line 4 of Algo.~\ref{algo:new:ogd:server}).  
Next, to leverage historical information, the server averages (i.e., averages) the past estimates \( \widehat{\theta}^{(s)} \) and uses the resulting \( \widetilde{\theta}^{(t_c+1)} \) as the updated estimate \( \theta_{\text{sync}} \) in round \( t \) (lines 5 of Algo.~\ref{algo:new:ogd:server}). 
In addition, the Central Server also updates the global information matrix $W_{\text{sync}}$ using the summation of the individual information matrices $W_{\text{new},i}$ (lines 6 of Algo.~\ref{algo:new:ogd:server}).
Finally, in line 7 of Algo.~\ref{algo:new:ogd:server}, the Central Server broadcasts $\theta_{\text{sync}}$, $W_{\text{sync}}$ and $\widehat{\theta}=\widehat{\theta}^{(t_c+1)}$ to all agents.
In the subsequent $\tau$ iterations, each agent $i$ uses $\theta_{\text{sync}}$ and $W_{\text{sync}}$ for arm pair selection (lines 4-5 of Algo.~\ref{algo:new:ogd:agent}), and uses $\widehat{\theta}=\widehat{\theta}^{(t_c+1)}$ to calculate the local gradients $\nabla l_{\text{new}}^{i}$ (line 8 of Algo.~\ref{algo:new:ogd:agent}).


\textbf{Communication Efficiency.}
Of note, our \texttt{FLDB-OGD} algorithm \emph{only needs one communication round after every $\tau\geq1$ iterations}. 
In contrast, the vanilla \texttt{FLDB-GD} algorithm (Sec.~\ref{subsec:algo:gd}) \emph{requires $M\geq 1$ communication rounds in every iteration}.
Therefore, \texttt{FLDB-OGD} is substantially more communication-efficient and hence more practical.
The parameter $\tau$ (i.e., the number of local update iterations) incurs a principled trade-off between the performance and communication efficiency of \texttt{FLDB-OGD}, which we will demonstrate both theoretically (Sec.~\ref{sec:theory}) and empirically (Sec.~\ref{sec:experiments}).

\vspace{-2mm}
\section{Theoretical Analysis}
\label{sec:theory}
\vspace{-2mm}

\subsection{Analysis of the Vanilla \texttt{FLDB-GD} Algorithm}
\label{subsec:theory:gd}
\vspace{-1mm}
We make the following assumption in our analysis of \texttt{FLDB-GD}.
\begin{assumption}\label{assumption:zero:gradient}
We assume that in every iteration of \texttt{FLDB-GD}, after $M$ rounds of gradient exchange (lines 8-11 of Algo.~\ref{algo:new:gd:agent}), 
the resulting $\theta_{\text{sync}}=\theta^{(M)}$ can exactly minimize \eqref{eq:loss:func:fed}, i.e., $\nabla\mathcal{L}_t^{\text{fed}}(\theta^{(M)})=0$. 
\end{assumption}
This assumption is commonly adopted by previous works on (non-federated) generalized linear bandits and dueling bandits \citep{li2017provablyoptimalalgorithmsgeneralized,NeurIPS21_saha2021optimal,ICML22_bengs2022stochastic}.
That is, it is a common practice in the literature to assume that the maximum likelihood estimation of the parameter $\theta$ is obtained exactly.
However, in the vanilla \texttt{FLDB-GD} algorithm, satisfying this assumption may necessitate a significantly large number of communication rounds between the agents and the Central Server.
Of note, the analysis of the practical \texttt{FLDB-OGD} algorithm (Sec.~\ref{subsec:theory:ogd}) does not require this assumption.

Denote $\theta_t\triangleq {\arg\min}_{\theta'} \mathcal{L}_t^{\text{fed}}(\theta')$, i.e., $\theta_t$ is the minimizer of the loss function $\mathcal{L}_t^{\text{fed}}$.
Note that according to Assumption \ref{assumption:zero:gradient}, we have that $\theta_{\text{sync}}= \theta_{t}$.
In our proof of the regret upper bound for \texttt{FLDB-GD}, a crucial step is to derive the following concentration bound.
\begin{lemma}\label{lemma:concentration:theta}
Let $\delta\in(0,1)$,
    $\beta_t \triangleq \sqrt{2\log(1/\delta) + d\log\left( 1 + tN\kappa_{\mu}/(d\lambda) \right)}$.
Then with probability of at least $1-\delta$, for all $t=1,\ldots,T$, we have that $\norm{\theta - \theta_{t}}_{V_t} \leq \frac{\beta_t}{\kappa_\mu}$.
\end{lemma}
Lemma \ref{lemma:concentration:theta} suggests that in every iteration $t$ of \texttt{FLDB-GD}, after running $M$ rounds of gradient exchange (lines 8-11 of Algo.~\ref{algo:new:gd:agent}), 
the resulting $\theta_{\text{sync}}=\theta_t$ is an accurate approximation of the groundtruth parameters $\theta$.
The regret of \texttt{FLDB-GD} can then be upper-bounded by the following theorem.
\begin{theorem}[\texttt{FLDB-GD}]
\label{theorem:gd}
With probability at least $1-\delta$, the overall regrets of all agents in all iterations satisfy:
\begin{equation}
\begin{split}
R_{T,N} &\leq 2N d\log \left( 1 + T N \kappa_{\mu}/(d\lambda) \right)  + 3 \sqrt{2} \frac{\beta_T}{\kappa_\mu} \sqrt{T 2 d\log \left( 1 + T N\kappa_{\mu}/(d\lambda) \right)}
\end{split}
\end{equation}
Ignoring all log factors, we have that
$
    R_{T,N} = \widetilde{O}\left(N d + \frac{d}{\kappa_\mu} \sqrt{T}\right)
$.
\end{theorem}
The regret upper bound for \texttt{FLDB-GD} (Theorem \ref{theorem:gd}) is sub-linear in $T$.
Theorem \ref{theorem:gd} suggests that the average regret of $N$ agents in our \texttt{FLDB-GD} algorithm is upper-bounded by $\widetilde{O}\left(d + \frac{d}{N \kappa_\mu} \sqrt{T}\right)$, which becomes smaller with a larger number of agents $N$.
This demonstrates \emph{the benefit of collaboration}, because it shows that if a larger number of agents $N$ join the federation of our algorithm, they are guaranteed (on average) to achieve a smaller regret compared with running their contextual dueling bandit algorithms in isolation (i.e., $N=1$).
When there is a single agent, the regret upper bound from Theorem \ref{theorem:gd} reduces to $\widetilde{O}(d + d\sqrt{T}/\kappa_\mu) = \widetilde{O}(d\sqrt{T}/\kappa_\mu)$, which is of the same order as the classical contextual linear dueling bandit algorithm \cite{ICML22_bengs2022stochastic}.

As discussed in Sec.~\ref{subsec:algo:gd}, the vanilla \texttt{FLDB-GD} algorithm incurs excessive communication costs and is hence not practical.
In the next section, we analyze the regret of our practical \texttt{FLDB-OGD} algorithm.

\vspace{-1mm}
\subsection{Analysis of the Practical \texttt{FLDB-OGD} Algorithm}
\label{subsec:theory:ogd}
\vspace{-1mm}


For simplicity, we firstly analyze 
\texttt{FLDB-OGD} for $\tau=1$, and then generalize it to the case of $\tau>1$.

In the analysis of \texttt{FLDB-OGD}, we do not require Assumption \ref{assumption:zero:gradient}. In other words, the $\theta_{\text{sync}}=\widetilde{\theta}^{(t)}$ returned by Algo.~\ref{algo:new:ogd:server} is no longer the minimizer of $\mathcal{L}_t^{\text{fed}}$ in \eqref{eq:loss:func:fed:ogd}.
However, here we prove that as long as the number of agents $N$ is sufficiently large, 
the difference between $\theta_{\text{sync}}=\widetilde{\theta}^{(t)}$ and $\theta_t$ is still bounded.
To begin with, the following proposition proves that 
as long as the number of agents $N$ is large enough, 
the loss function $f_s^{\text{fed}}(\theta^{\prime})$ \eqref{eq:loss:func:fed:local:sum:ogd} in every iteration $s$ is strongly convex.
\begin{proposition}[$\alpha$-strongly convex]\label{prop:alpha:convex}
    Denote $\mathbbm{B}_{\eta} := \{ \theta^{\prime}: \|\theta^{\prime}-\theta\| \leq \eta\}$.
    For a constant $\alpha > 0$, let 
    \begin{equation*}
        N \geq \big( (C_1 \sqrt{d} + C_2 \sqrt{\log  2 / \delta }) / \lambda_f \big)^2 + 2\alpha / (\kappa_{\mu} \lambda_f),
    \end{equation*}
    where $C_1$ and $C_2$ are two universal constants.
     Then $f_s^{\text{fed}} (\theta^{\prime})$ is an $\alpha$-strongly convex function in $\mathbbm{B}_{3r}$, with probability at least $1-\delta$. 
\end{proposition}
Recall that every $f_s^{\text{fed}}(\theta^{\prime})=\sum_{i=1}^{N} l_s^{i}(\theta')$ corresponds to an individual loss function in our overall loss function $\mathcal{L}_t^{\text{fed}}(\theta')$ \eqref{eq:loss:func:fed:ogd} during the \texttt{FLDB-OGD} algorithm.
Therefore, 
we can make use of 
Proposition \ref{prop:alpha:convex} and the convergence result of OGD for strongly convex functions \cite{hazan2016introduction} to derive the following lemma.
\begin{lemma}\label{lemma:ogd:convergence}
    Under the same condition on $N$ as Proposition \ref{prop:alpha:convex}, with probability at least $1-\delta$, the following holds for all $t\geq 1$: 
         $\norm{\widetilde{\theta}^{(t)} - \theta_t}_{V_t} \leq \frac{N \sqrt{N + \frac{\lambda}{\kappa_{\mu}}}}{\alpha} \sqrt{1+\log t}$.
\end{lemma}
Lemma \ref{lemma:ogd:convergence} shows that the difference between $\theta_{\text{sync}}=\widetilde{\theta}^{(t)}$ (returned by Algo.~\ref{algo:new:ogd:server}) and $\theta_t\triangleq {\arg\min}_{\theta'} \mathcal{L}_t^{\text{fed}}(\theta')$ is bounded.
As a result, Lemma \ref{lemma:ogd:convergence}, combined with Lemma \ref{lemma:concentration:theta} which has provided an upper bound on the difference between $\theta_t$ and $\theta$, allows us to bound the difference between $\widetilde{\theta}^{(t)}$ and $\theta$.
That is, the $\theta_{\text{sync}}=\widetilde{\theta}^{(t)}$ returned by Algo.~\ref{algo:new:ogd:server} is an accurate estimation of the groundtruth linear function parameter $\theta$.


With these supporting lemmas, we can derive an upper bound on the cumulative regret of \texttt{FLDB-OGD} when $\tau=1$ (i.e., if a communication round occurs after every iteration).
\begin{theorem}[\texttt{FLDB-OGD} with $\tau=1$]\label{theorem:ogd}
    Ignoring all log factors, with probability of at least $1-\delta$, the overall regret of \texttt{FLDB-OGD} (when $\tau=1$) can be bounded by
    \begin{equation}
    \begin{split}
    R_{T,N} = \widetilde{O}\big(N d + \frac{d}{\kappa_\mu} \sqrt{T} +  \frac{N^{\frac{3}{2}} \sqrt{d}}{\alpha} \sqrt{T}\big)
    \end{split}
    \end{equation}
\end{theorem}
The regret upper bound of our \texttt{FLDB-OGD} algorithm is also sub-linear in $T$.
Compared with \texttt{FLDB-GD} (Theorem \ref{theorem:gd}), the regret upper bound for \texttt{FLDB-OGD} has an additional term of $\widetilde{O}(\frac{N^{3/2} \sqrt{d}}{\alpha}\sqrt{T})$.
This is the loss we suffer for not ensuring that the $\theta_{\text{sync}}=\widetilde{\theta}^{(t)}$ returned by Algo.~\ref{algo:new:ogd:server} could achieve the minimum of $\mathcal{L}_t^{\text{fed}}$ \eqref{eq:loss:func:fed:ogd}.
On the other hand, by paying this cost in terms of regrets, \texttt{FLDB-OGD} gains considerably smaller communication costs than \texttt{FLDB-GD}.
This is because \texttt{FLDB-OGD} only needs a single round of communication after every $\tau$ iterations, whereas \texttt{FLDB-GD} requires a large number of communication rounds in every iteration to find the minimum of $\mathcal{L}_t^{\text{fed}}(\theta')$ \eqref{eq:loss:func:fed}.

Unlike the theoretical result for \texttt{FLDB-GD} (Theorem \ref{theorem:gd}), the regret upper bound in (Theorem \ref{theorem:ogd}) is not improved as the number $N$ of agents is increased.
This is because in the proof of Lemma \ref{lemma:ogd:convergence} (App.~\ref{app:sec:proof:lemma:ogd:convergence}), we have made use of the convergence of OGD for strongly convex functions  \cite{hazan2016introduction}, which requires an upper bound on the expected norm of the gradient: $G^2 \geq E\|\nabla f_{s}\|^2$.
We have shown that the worst-case choice of $G$ is $G=N$ (see \eqref{eq:G:upper:bound} in App.~\ref{app:sec:proof:lemma:ogd:convergence}), which contributed a dependency of $N$ to the last term in the regret upper bound in Theorem \ref{theorem:ogd}.
In other words, if we additionally assume that there exists an upper bound $G$ on expected gradient norm which is independent of $N$, then the last term in Theorem \ref{theorem:ogd} can be replaced by $\frac{G\sqrt{N} \sqrt{d}}{\alpha} \sqrt{T}$. This would then make our average regret upper bound (averaged over $N$ agents) become tighter with a larger number $N$ of agents.

Now we analyze the regret of our practical \texttt{FLDB-OGD} algorithm when $\tau>1$, i.e., when a communication round occurs after multiple iterations.
\begin{proposition}[\texttt{FLDB-OGD} with $\tau>1$]\label{prop:regret:local:update}
    Ignoring all log factors, with probability of at least $1-\delta$, the overall regret of \texttt{FLDB-OGD} (when $\tau>1$) can be bounded by
    \begin{equation}
    R_{T,N} = \widetilde{O}\big(\tau N d + \frac{d}{\sqrt{\tau} \kappa_\mu} \sqrt{T} +  \tau\frac{N^{\frac{3}{2}} \sqrt{d}}{\alpha} \sqrt{T}\big)
    \end{equation}
\end{proposition}
\textbf{Trade-off between Regrets and Communication.}
As long as the number of agents $N$ is large enough (i.e., ensured by Prop.~\ref{prop:alpha:convex}), the second term in the regret upper bound from Prop.~\ref{prop:regret:local:update} is dominated by the other two terms.
This implies that a larger number of local updates $\tau$ (which indicates less frequent communication between the agents and the Central Server)
results in a worse regret upper bound for \texttt{FLDB-OGD}.
On the other hand, a larger $\tau$ reduces the total number of communication rounds by a factor of $1/\tau$. This induces a \emph{theoretical trade-off between regrets and communication complexity}, which is also validated by our empirical experiments (Sec.~\ref{sec:experiments}).

\vspace{-2mm}
\section{Experiments}
\label{sec:experiments}
\vspace{-2mm}

Here we evaluate the performance of our algorithms using both synthetic and real-world experiments. 
We let $\tau=1$ unless specified otherwise.

\textbf{Experimental Settings.}
In the synthetic experiments, we generate the groundtruth linear function parameters $\theta$ in each experiment by randomly sampling from the standard Gaussian distribution. 
In each round, every agent receives $K$ arms (i.e., contexts) randomly generated from the standard Gaussian distribution $\mathcal{N}(0, I)$ with dimension $d$. 
For \texttt{FLDB-GD}, we adopt the \texttt{LBFGS} method from \texttt{PyTorch}~\citep{paszke2019pytorchimperativestylehighperformance} to efficiently find the minimizer of the loss function~\eqref{eq:loss:func:fed} in every iteration.
All figures in this section plot the \emph{average cumulative regret} of all $N$ agents up to a given iteration, which allows us to inspect the benefit brought by the federated setting to an agent on average. 
To verify the benefit of collaboration in the federated setting, we compare our algorithms with the single-agent baseline of linear dueling bandits (\texttt{LDB}) \cite{ICML22_bengs2022stochastic}, 
in which every agent simply performs a local linear dueling bandit algorithm in isolation.
For real-world evaluations, we adopt the MovieLens dataset~\cite{movielens}, following the experimental setup from~\cite{wang2023onlineclusteringbanditsmisspecified}. 
More details on the experimental settings are deferred to
App.~\ref{app:sec:more:exp:details}.

\begin{figure*}[t]
\vspace{-2mm}
     \centering
     \begin{tabular}{cccc}
         \hspace{-0.6cm} 
         \includegraphics[width=0.25\linewidth]{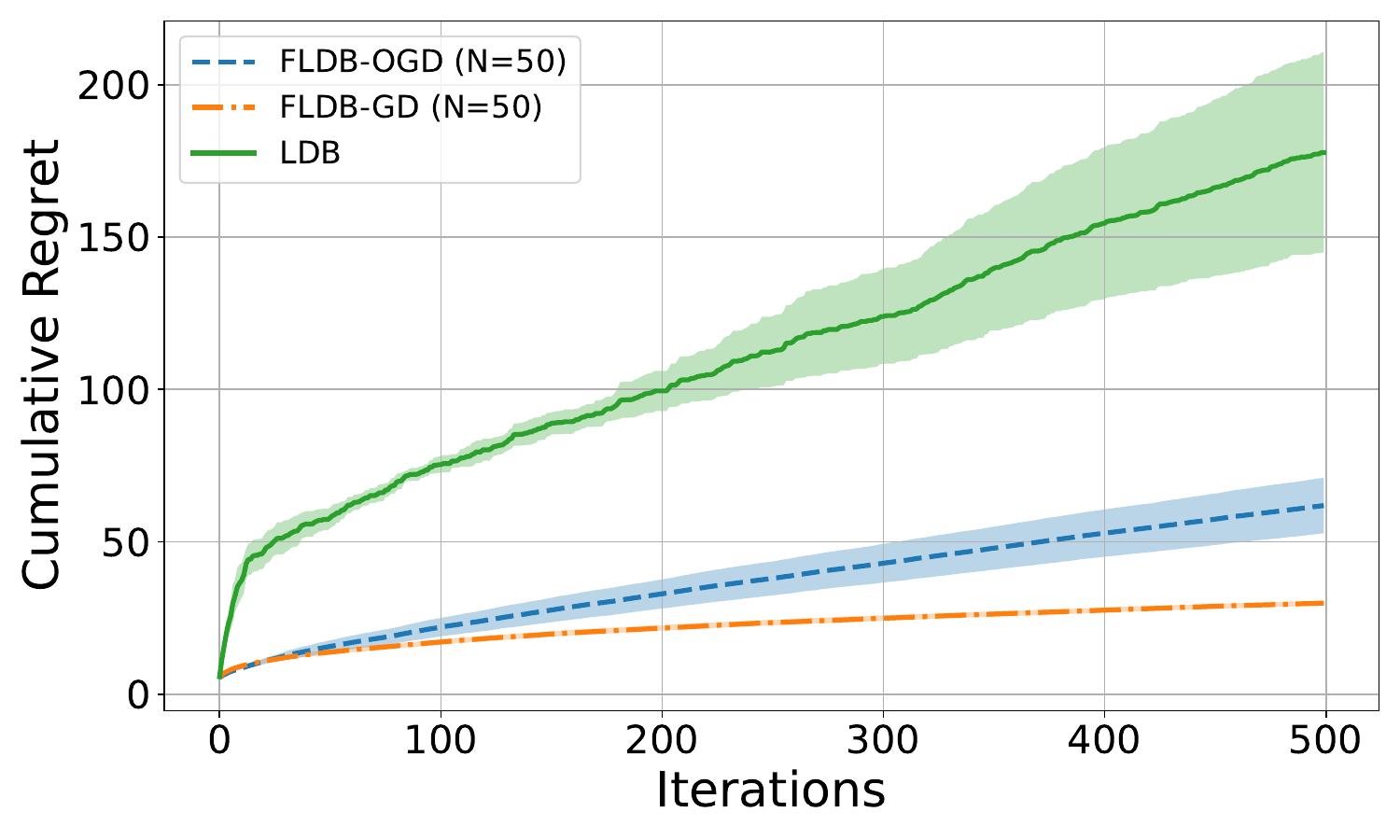} & 
         
         \hspace{-0.4cm} 
         \includegraphics[width=0.25\linewidth]{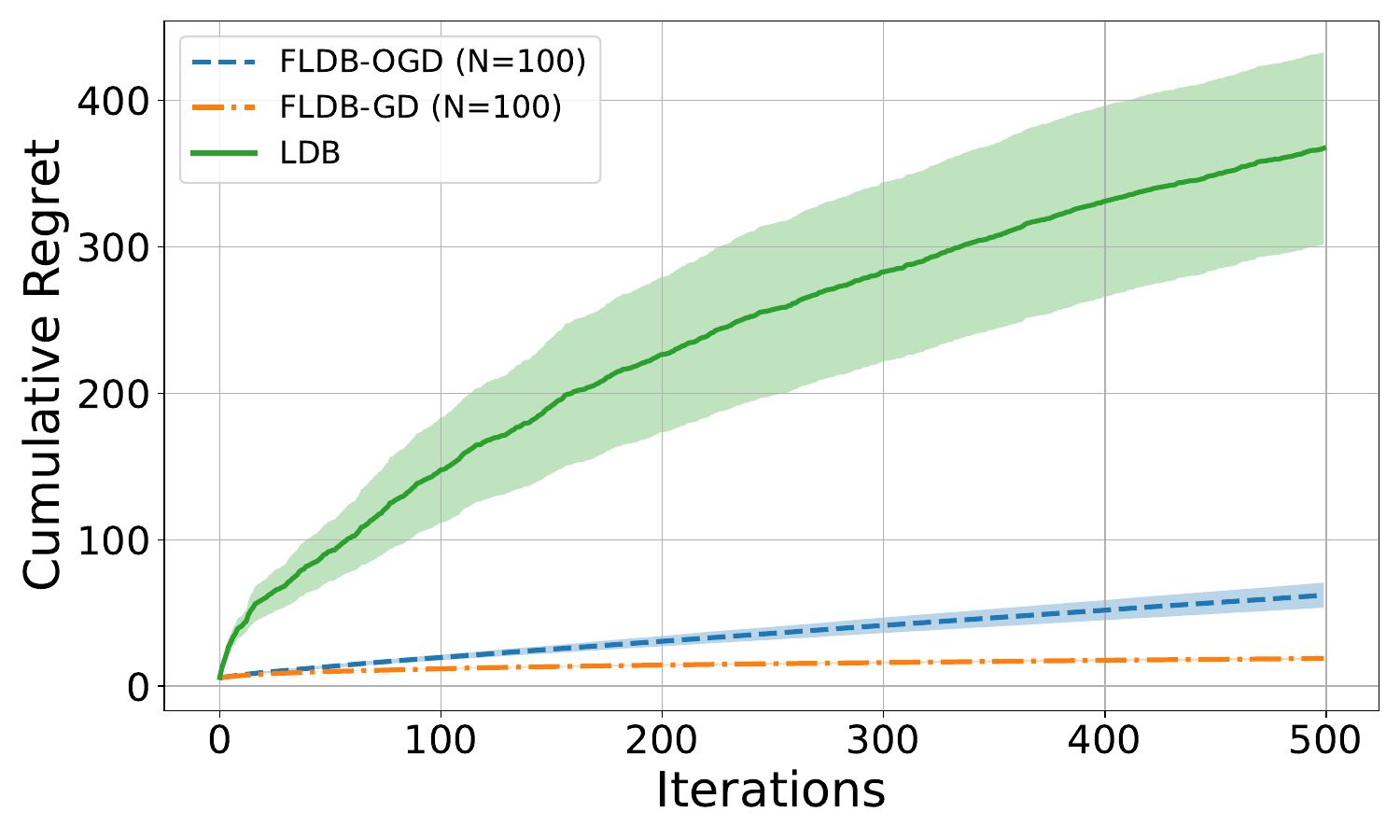}& 
         
         \hspace{-0.4cm}
         \includegraphics[width=0.25\linewidth]{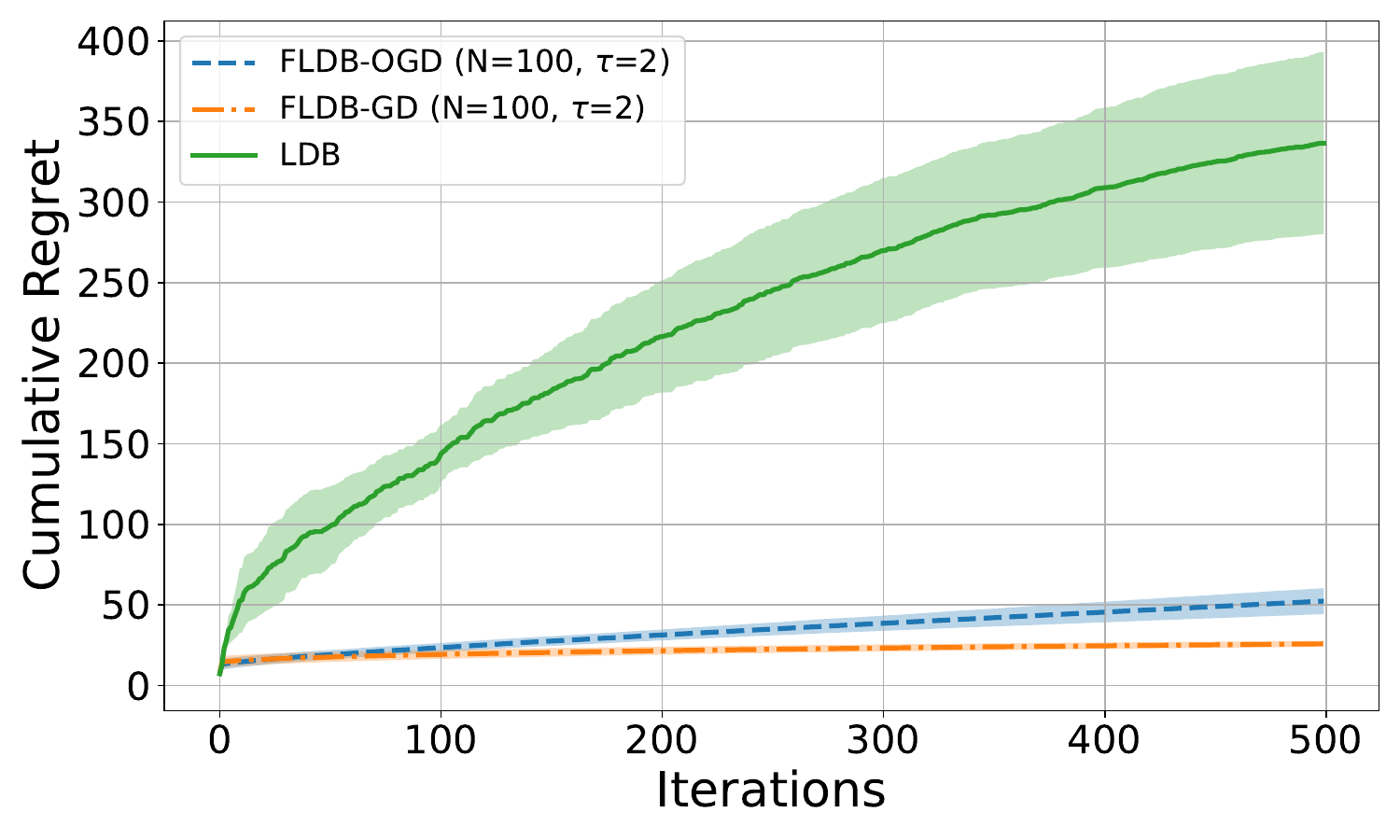} &  
         
         \hspace{-0.4cm}   
         \includegraphics[width=0.25\linewidth]{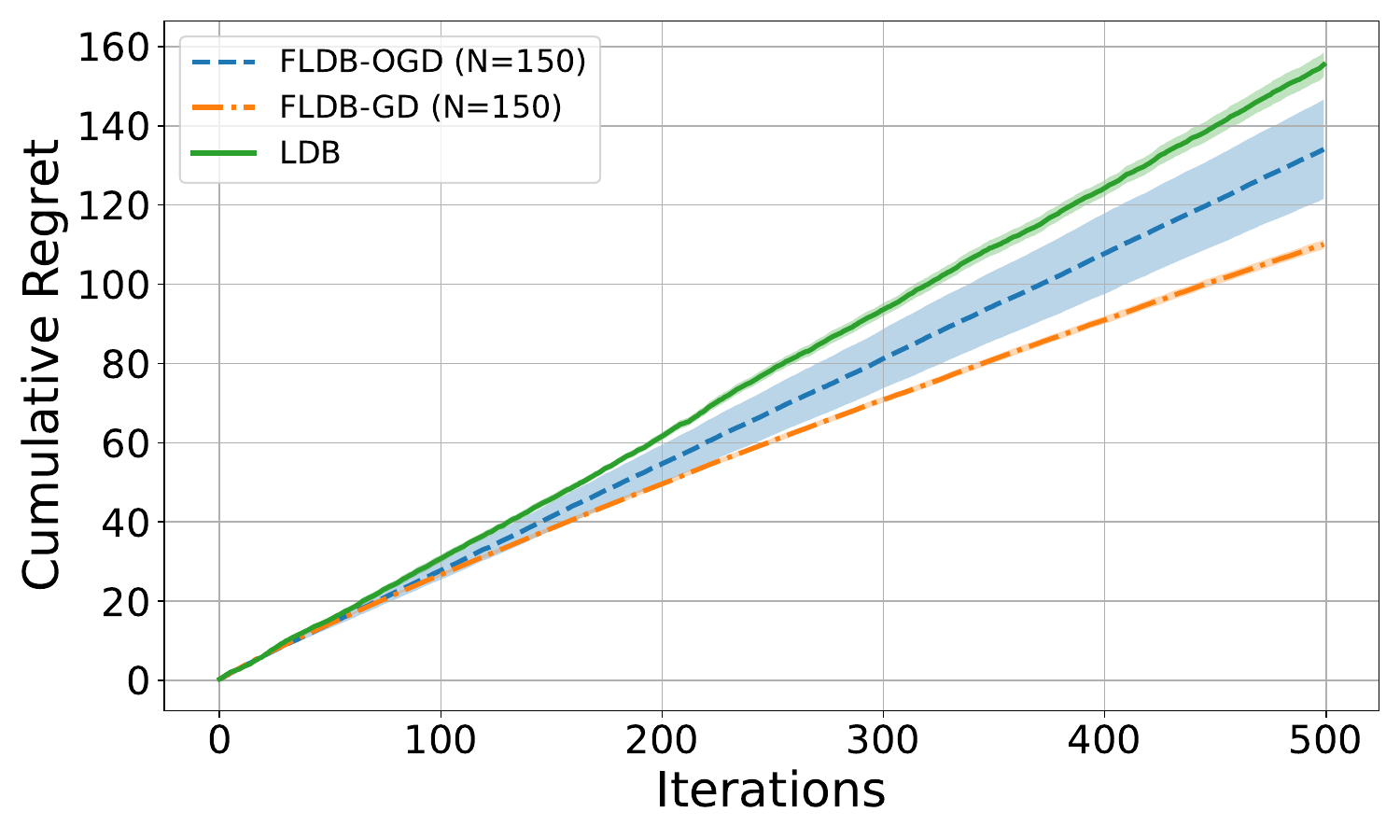}
         \\
         {\hspace{-0.6cm} {\footnotesize(a)~$N=50$}} & {\hspace{-0.4cm} {\footnotesize(b)~$N=100$}} & {\hspace{-0.4cm}{\footnotesize(c)~$N=100, \tau=2$}} &
         {\hspace{-4mm} {\footnotesize(d)~MovieLens}} 
         \\
         {\hspace{-0.6cm} {\footnotesize$(K=10, d=5)$}} & {\hspace{-0.4cm} {\footnotesize$(K=10, d=5)$}} & {\hspace{-0.4cm}{\footnotesize$(K=10, d=5)$}} &
         {\hspace{-4mm} {\footnotesize$(K=5, d=10)$}} 
     \end{tabular}
\vspace{-2.8mm}
     \caption{
    Cumulative regret for different methods with varying numbers of agents: (a) $N=50$, (b) $N=100$, 
    (c) $N=100,\tau=2$ under the number of arms $K=10$ and dimension $d=5$, and (d) MovieLens Dataset with $K=5$ and $d=10$.
}

     \label{fig:exp:synth_1}
\vspace{-1.4mm}
\end{figure*}

\begin{figure*}[t]
\vspace{-2mm}
     \centering
     \begin{tabular}{cccc}
         \hspace{-0.6cm} 
         \includegraphics[width=0.25\linewidth]{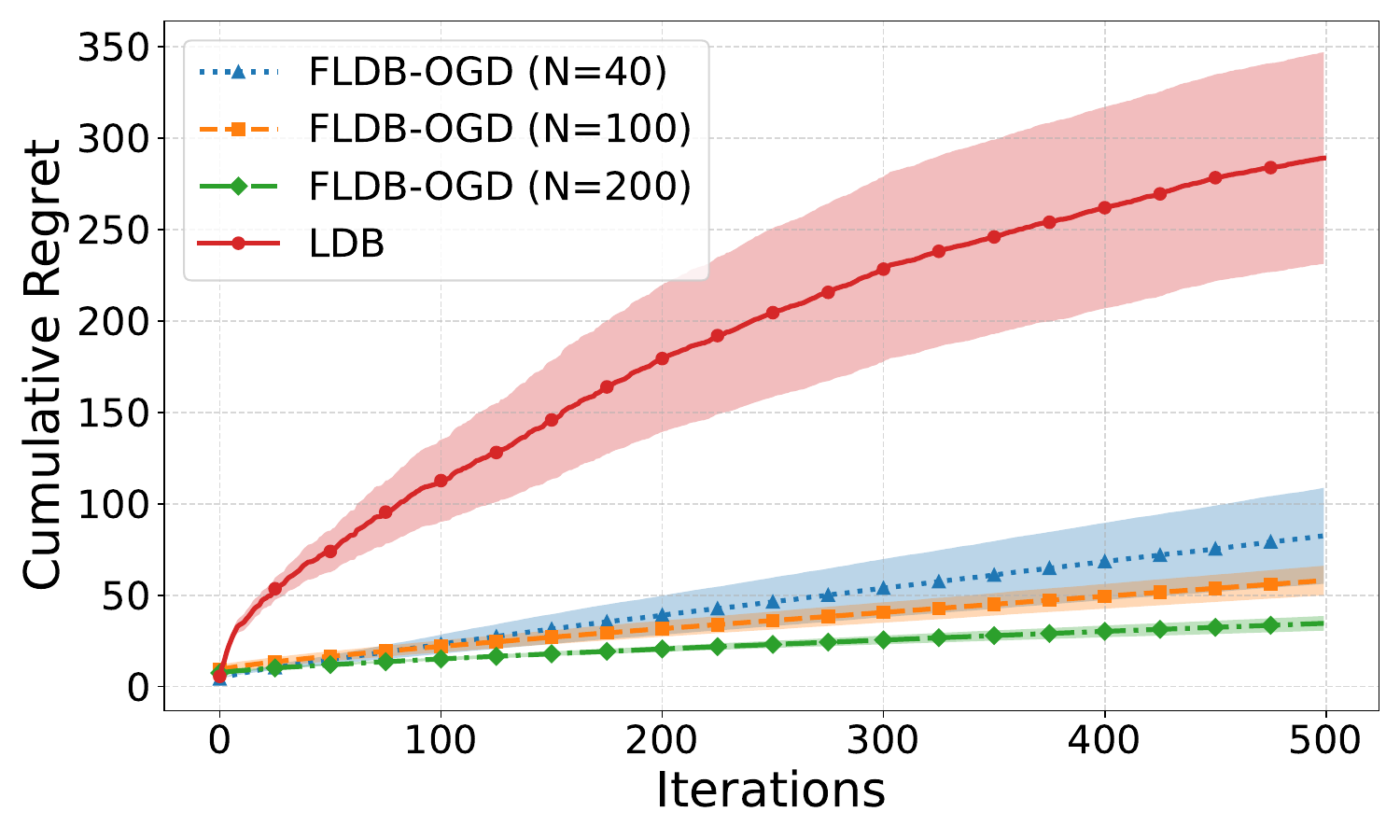} & 
         \hspace{-0.4cm} 
         \includegraphics[width=0.25\linewidth]{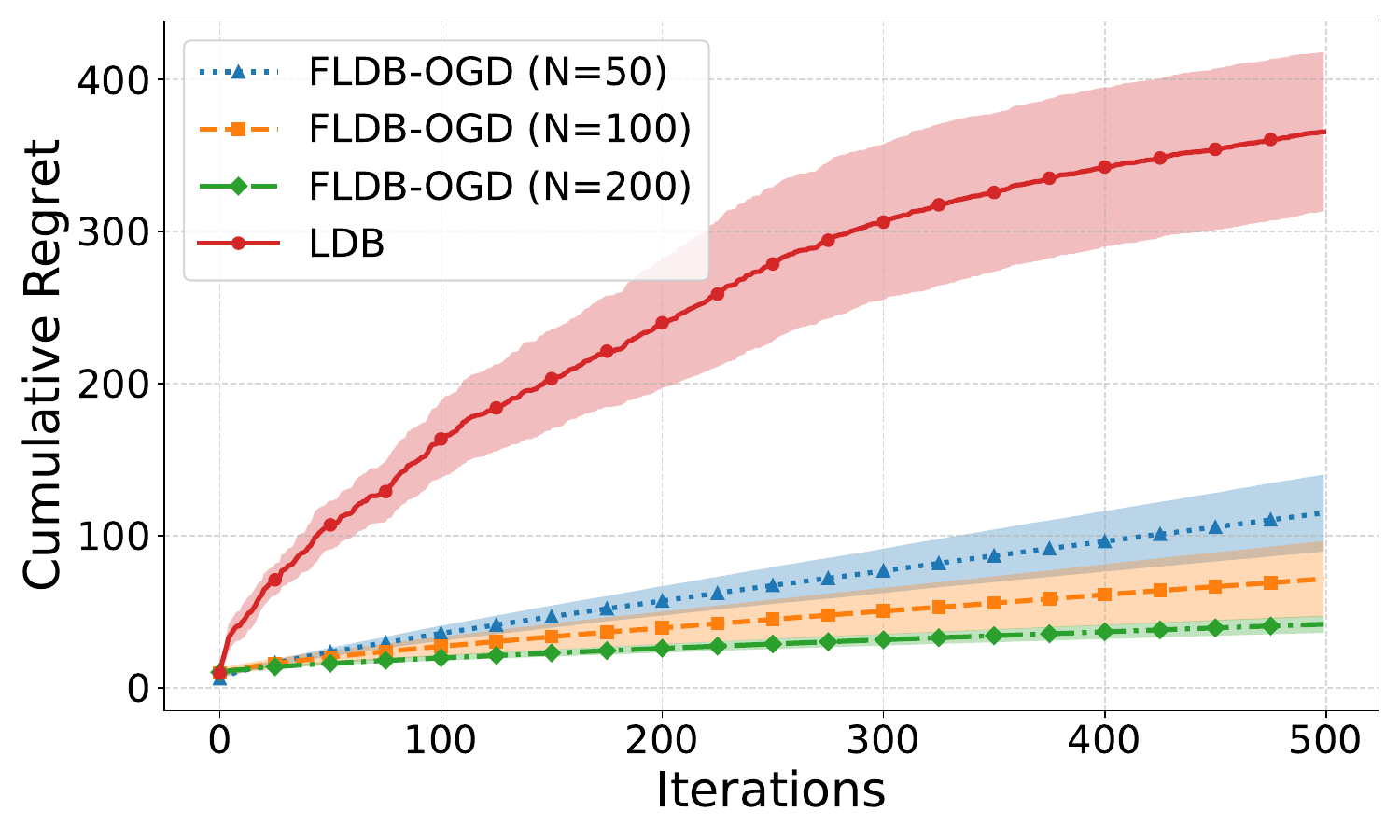}
          & 
          \hspace{-0.4cm}
         \includegraphics[width=0.25\linewidth]{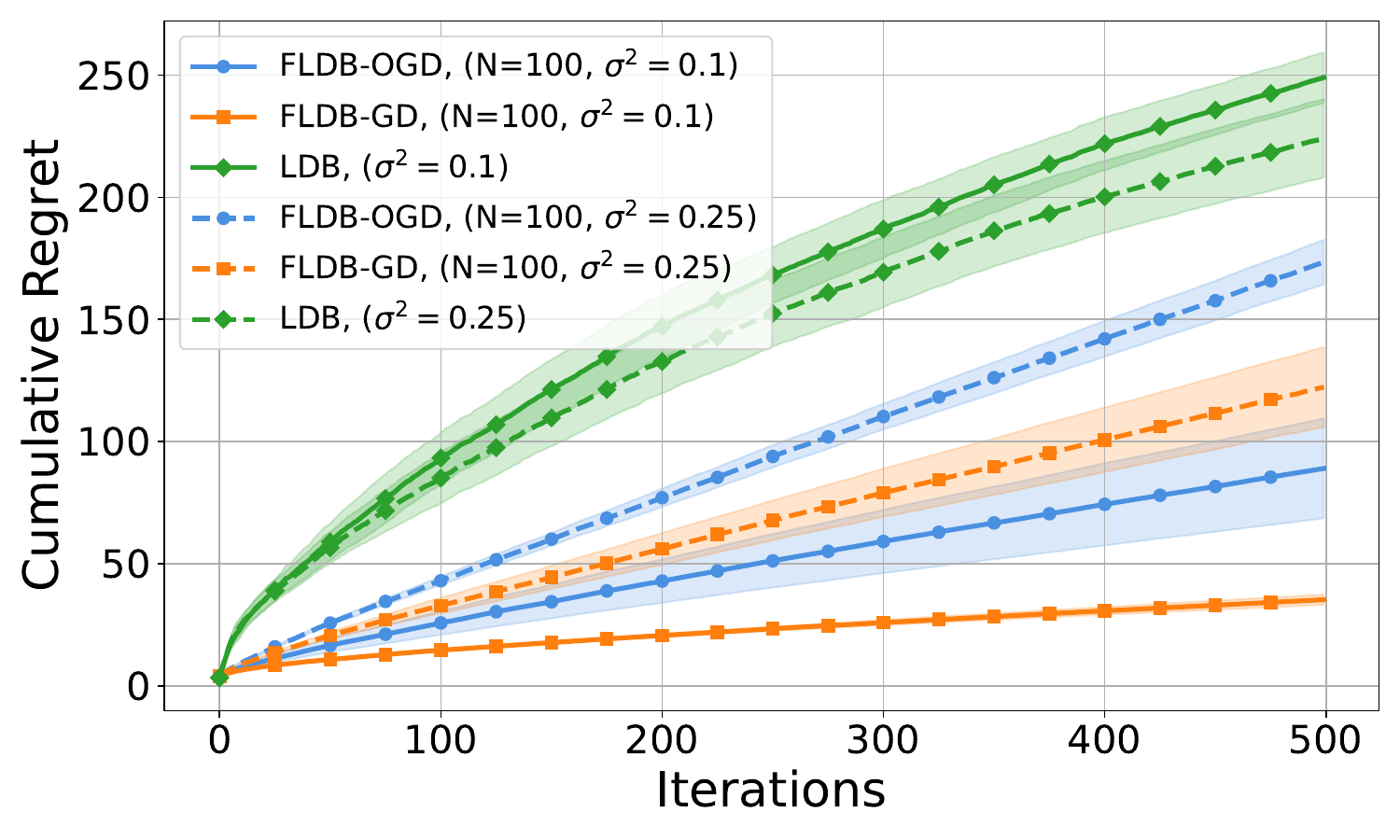} &  
         
         \hspace{-0.7cm}
         \includegraphics[width=0.25\linewidth]{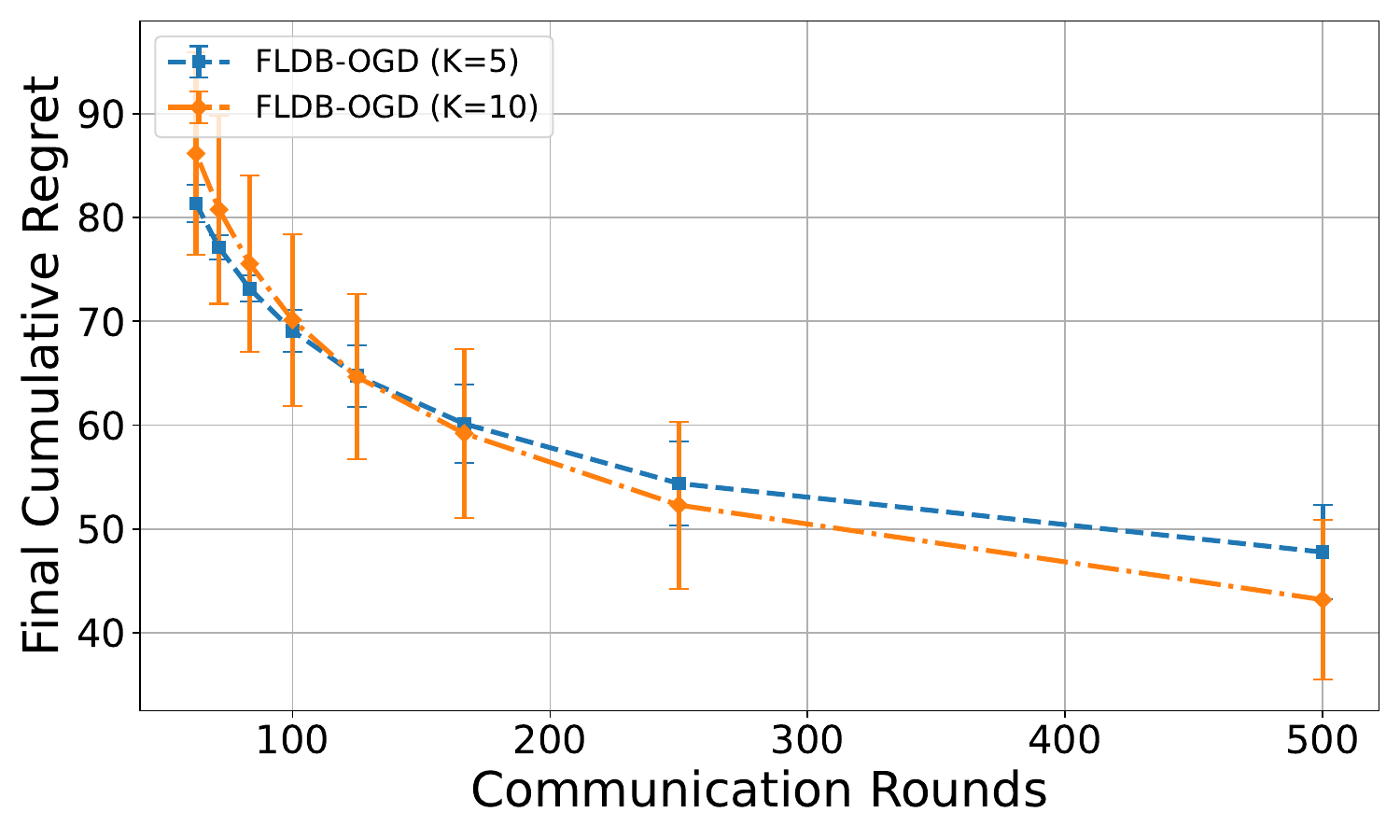}
         \\
         {\hspace{-0.6cm} {\footnotesize(a)~\texttt{FLDB-OGD}}} & {\hspace{-0.4cm} {\footnotesize(b)~\texttt{FLDB-OGD}}} & {\hspace{-0.4cm}{\footnotesize(c)~Heterogeneous }} &
         {\hspace{-4.5mm} {\footnotesize(d)~Communication vs.~Regrets}} \\
         {\hspace{-0.6cm} {\footnotesize $(K=10, d=5)$}} & {\hspace{-0.4cm} {\footnotesize $(K=50, d=5)$}} & {\hspace{-0.4cm}{\footnotesize $(K=5, d=5)$}} &
         {\hspace{-4mm} {\footnotesize $(N=100, d=5)$}} 
     \end{tabular}
\vspace{-2.8mm}
    \caption{
    Cumulative regret of \texttt{FLDB-OGD} under different settings. 
    (a)–(b): Impact of the number of agents with $K=10,50$ and $d=5$. 
    (c): Performance under heterogeneous rewards with $K=5$, $d=5$. 
    (d): Final regret versus number of communication rounds with $N=100$, $d=5$.
    }
     \label{fig:exp:ablation}
\vspace{-4.5mm}
\end{figure*}


\textbf{Results.}
In the synthetic experiments, we evaluate our \texttt{FLDB-OGD} algorithm, as well as the vanilla \texttt{FLDB-GD} algorithm and the \texttt{LDB} baseline, under the same settings of the number of agents $N$, the number of arms $K$ and the input dimension $d$.
To begin with, we fix \( K = 10 \) and \( d = 5 \), and evaluate the impact of different number of agents $N$ and different update period $\tau$. In the MovieLens experiments, we set \( K = 5 \) and \( d = 10 \). The results, presented in Figs.~\ref{fig:exp:synth_1}a--\ref{fig:exp:synth_1}d, show that both \texttt{FLDB-GD} and \texttt{FLDB-OGD} consistently outperform the single-agent baseline \texttt{LDB} across all settings.
Moreover, \texttt{FLDB-GD} achieves lower cumulative regret than \texttt{FLDB-OGD}, which is consistent with our theoretical analysis in Section~\ref{subsec:theory:ogd}, as \texttt{FLDB-GD} enjoys a tighter regret bound compared to \texttt{FLDB-OGD}. However, as we have pointed out in Sec.~\ref{subsec:algo:ogd} and Sec.~\ref{subsec:theory:ogd}, this performance advantage of \texttt{FLDB-GD} in terms of regret comes at the expense of \emph{significantly increased communication costs}. 

In addition, we conduct an ablation study to further examine the behavior of our \texttt{FLDB-OGD} algorithm under different experimental conditions. Specifically, we analyze the effects of (i) the number of agents ($N$), (ii) reward heterogeneity, and (iii) the number of communication rounds.

\textbf{Impact of the Number of Agents $N$.}
Here we examine how the number of agents $N$ affects the performance of the practical \texttt{FLDB-OGD} algorithm, while fixing the values of $K$ and $d$.
As shown in Fig.~\ref{fig:exp:ablation}a and b, the cumulative regret of \texttt{FLDB-OGD} decreases as the number of agents $N$ becomes larger, indicating that the collaboration among more agents indeed leads to improved performance.
These results validate the effectiveness of our \texttt{FLDB-OGD} algorithm in leveraging collaborative information from multiple agents, and offer practical motivation for encouraging more agents to participate in the federated learning process.
The observed empirical improvement of \texttt{FLDB-OGD} with larger $N$ does not follow directly from the theoretical guarantee provided in Theorem~\ref{theorem:ogd}.
This discrepancy suggests that the theoretical bound in Theorem~\ref{theorem:ogd} may be overly conservative. Meanwhile, it provides justifications for our additional assumption at the end of Section~\ref{subsec:theory:ogd}, which posits that the gradient norm bound $G$ is independent of $N$. Under this assumption, the regret bound for \texttt{FLDB-OGD} indeed improves with larger $N$, leading to better agreement between theory and practice.

\textbf{Heterogeneous Agents.}
To evaluate the performance of our algorithms when the reward functions of different agents are heterogeneous, we introduce variability into the reward functions of different agents by perturbing the global parameters $\theta^*$.
Specifically, for each agent $i$, we define a personalized parameter $\theta^*_i = \theta^* + \epsilon$, where the perturbation $\epsilon$ is sampled from a multivariate Gaussian distribution $\mathcal{N}(0, \sigma^2 I)$. We vary the standard deviation $\sigma$ from 0.1 to 0.5 to control the degree of heterogeneity.
Fig.~\ref{fig:exp:ablation}c presents the results for $\sigma^2 = 0.1$ and $\sigma^2 = 0.25$. For each fixed value of $\sigma$, we observe that the trends are consistent with those in the homogeneous synthetic setting (Fig.~\ref{fig:exp:synth_1}). Both \texttt{FLDB-OGD} and \texttt{FLDB-GD} incur higher cumulative regrets as the level of degree of heterogeneity increases.
Of note, \texttt{FLDB-OGD} maintains favorable performance in the presence of moderate heterogeneity, demonstrating its robustness under heterogeneous 
reward functions of different agents. 


\textbf{Regret-Communication Trade-off.}
To investigate the effect of communication frequency on the performance of \texttt{FLDB-OGD}, we vary the number of local updates $\tau$ in each iteration and report the resulting final regret after $T=500$ iterations.
Given the values of $T$ and $\tau$, the total number of communication rounds is given by $T / \tau$.
Therefore, for a fixed $T$, a smaller $\tau$ corresponds to more frequent communication, while a larger $\tau$ reduces the number of communication rounds. 
Fig.~\ref{fig:exp:ablation}d illustrates the relationship between communication frequency and the final regret. 
These results corroborate our theoretical analysis in Prop.~\ref{prop:regret:local:update}: a larger number of communication rounds (i.e., smaller $\tau$) leads to smaller regrets.

\vspace{-2mm}
\section{Related work}
\vspace{-2mm}
\textbf{Federated Bandits.}
Recent studies have extended the classical $K$-armed bandit problem to the federated setting. 
\cite{li2022privacy,li2020federated} introduced privacy-preserving federated $K$-armed bandits in centralized and decentralized settings, respectively. \cite{shi2021federated} formulated a global bandit model where arm rewards are averaged across agents, which was later extended to incorporate personalization \citep{shi2021federatedwithpersonalization}. 
For federated linear contextual bandits, \cite{wang2019distributed} proposed a distributed algorithm using sufficient statistics to compute the Linear UCB policy, which was later extended to incorporate differential privacy \citep{dubey2020differentially}, agent-specific contexts \citep{huang2021federated}, and asynchronous communication \citep{li2021asynchronous}. 
Federated kernelized 
and neural bandits have been developed for hyperparameter tuning \citep{dai2020federated,dai2021differentially,dai2022federated}.
In addition, many recent works have extended federated bandits to various settings and applied them to solve different real-world problems \cite{li2022communicationkcb,li2022communicationglb,zhu2021federated,ciucanu2022samba,blaser2024federated,fan2024federated,yang2024federated,solanki2024fairness,fourati2024federated,wang2024towards,li2024federated,wei2024incentivized,li2024fedconpe}.


\textbf{Dueling Bandits.}
Dueling bandits have received considerable attention in recent years
\citep{ICML09_yue2009interactively,ICML11_yue2011beat,JCSS12_yue2012k,WSDM14_zoghi2014relative,ICML14_ailon2014reducing,ICML14_zoghi2014relative,COLT15_komiyama2015regret,ICML15_gajane2015relative,UAI18_saha2018battle,AISTATS19_saha2019active,ALT19_saha2019pac,AISTATS22_saha2022exploiting,ICML23_zhu2023principled}.
To account for complicated real-world scenarios, a number of contextual dueling bandit algorithms have been developed which model the reward function using either a linear function
\citep{NeurIPS21_saha2021optimal,ICML22_bengs2022stochastic,arXiv24_li2024feelgood,ALT22_saha2022efficient,arXiv23_di2023variance} or a neural network \cite{verma2024neural}.

\vspace{-2mm}
\section{Conclusion and Limitation}
\vspace{-2mm}
\label{sec:conclusion}
We have introduced \texttt{FLDB-OGD}, the first federated linear dueling bandit algorithm, enabling multi-agent collaboration while preserving data privacy. 
By integrating online gradient descent with federated learning, our approach is able to effectively estimate the linear function parameters. Theoretical analysis guarantees a sub-linear cumulative regret bound, and empirical results demonstrate the benefits of collaboration and the regret-communication trade-off.
A potential limitation is that our methods are not applicable to dueling bandit problems with non-linear reward functions (when a known non-linear feature mapping is unavailable). We plan to address this challenge in future work by extending our methods to the frameworks of kernelized bandits and neural bandits.

\bibliography{reference}

\begin{thebibliography}{10}

\bibitem{NeurIPS21_saha2021optimal}
Aadirupa Saha.
\newblock Optimal algorithms for stochastic contextual preference bandits.
\newblock In {\em Proc. NeurIPS}, pages 30050--30062, 2021.

\bibitem{ALT22_saha2022efficient}
Aadirupa Saha and Akshay Krishnamurthy.
\newblock Efficient and optimal algorithms for contextual dueling bandits under realizability.
\newblock In {\em Proc. ALT}, pages 968--994, 2022.

\bibitem{ICML22_bengs2022stochastic}
Viktor Bengs, Aadirupa Saha, and Eyke H{\"u}llermeier.
\newblock Stochastic contextual dueling bandits under linear stochastic transitivity models.
\newblock In {\em Proc. ICML}, pages 1764--1786, 2022.

\bibitem{arXiv24_li2024feelgood}
Xuheng Li, Heyang Zhao, and Quanquan Gu.
\newblock Feel-good thompson sampling for contextual dueling bandits.
\newblock {\em arXiv:2404.06013}, 2024.

\bibitem{JCSS12_yue2012k}
Yisong Yue, Josef Broder, Robert Kleinberg, and Thorsten Joachims.
\newblock The k-armed dueling bandits problem.
\newblock {\em Journal of Computer and System Sciences}, pages 1538--1556, 2012.

\bibitem{lin2024prompt}
Xiaoqiang Lin, Zhongxiang Dai, Arun Verma, See-Kiong Ng, Patrick Jaillet, and Bryan Kian~Hsiang Low.
\newblock Prompt optimization with human feedback.
\newblock {\em arXiv preprint arXiv:2405.17346}, 2024.

\bibitem{ji2024reinforcement}
Kaixuan Ji, Jiafan He, and Quanquan Gu.
\newblock Reinforcement learning from human feedback with active queries.
\newblock {\em arXiv preprint arXiv:2402.09401}, 2024.

\bibitem{mcmahan2017communication}
Brendan McMahan, Eider Moore, Daniel Ramage, Seth Hampson, and Blaise~Aguera y~Arcas.
\newblock Communication-efficient learning of deep networks from decentralized data.
\newblock In {\em Artificial intelligence and statistics}, pages 1273--1282. PMLR, 2017.

\bibitem{shi2021federated}
Chengshuai Shi and Cong Shen.
\newblock Federated multi-armed bandits.
\newblock In {\em Proc. {AAAI}}, 2021.

\bibitem{wang2019distributed}
Yuanhao Wang, Jiachen Hu, Xiaoyu Chen, and Liwei Wang.
\newblock Distributed bandit learning: Near-optimal regret with efficient communication.
\newblock {\em arXiv preprint arXiv:1904.06309}, 2019.

\bibitem{NIPS11_abbasi2011improved}
Yasin Abbasi-Yadkori, D{\'a}vid P{\'a}l, and Csaba Szepesv{\'a}ri.
\newblock Improved algorithms for linear stochastic bandits.
\newblock In {\em Proc. NeurIPS}, pages 2312--2320, 2011.

\bibitem{dai2022federated}
Zhongxiang Dai, Yao Shu, Arun Verma, Flint~Xiaofeng Fan, Bryan Kian~Hsiang Low, and Patrick Jaillet.
\newblock Federated neural bandits.
\newblock {\em arXiv preprint arXiv:2205.14309}, 2022.

\bibitem{AS04_hunter2004mm}
David~R Hunter.
\newblock Mm algorithms for generalized bradley-terry models.
\newblock {\em Annals of Statistics}, pages 384--406, 2004.

\bibitem{Book_luce2005individual}
R~Duncan Luce.
\newblock {\em Individual choice behavior: A theoretical analysis}.
\newblock Courier Corporation, 2005.

\bibitem{ICML17_li2017provably}
Lihong Li, Yu~Lu, and Dengyong Zhou.
\newblock Provably optimal algorithms for generalized linear contextual bandits.
\newblock In {\em Proc. ICML}, pages 2071--2080, 2017.

\bibitem{ding2021efficientalgorithmgeneralizedlinear}
Qin Ding, Cho-Jui Hsieh, and James Sharpnack.
\newblock An efficient algorithm for generalized linear bandit: Online stochastic gradient descent and thompson sampling, 2021.

\bibitem{li2017provablyoptimalalgorithmsgeneralized}
Lihong Li, Yu~Lu, and Dengyong Zhou.
\newblock Provably optimal algorithms for generalized linear contextual bandits, 2017.

\bibitem{wu2020stochasticlinearcontextualbandits}
Weiqiang Wu, Jing Yang, and Cong Shen.
\newblock Stochastic linear contextual bandits with diverse contexts, 2020.

\bibitem{hazan2016introduction}
Elad Hazan et~al.
\newblock Introduction to online convex optimization.
\newblock {\em Foundations and Trends{\textregistered} in Optimization}, 2(3-4):157--325, 2016.

\bibitem{paszke2019pytorchimperativestylehighperformance}
Adam Paszke, Sam Gross, Francisco Massa, Adam Lerer, James Bradbury, Gregory Chanan, Trevor Killeen, Zeming Lin, Natalia Gimelshein, Luca Antiga, Alban Desmaison, Andreas Köpf, Edward Yang, Zach DeVito, Martin Raison, Alykhan Tejani, Sasank Chilamkurthy, Benoit Steiner, Lu~Fang, Junjie Bai, and Soumith Chintala.
\newblock Pytorch: An imperative style, high-performance deep learning library, 2019.

\bibitem{movielens}
F.~Maxwell Harper and Joseph~A. Konstan.
\newblock The movielens datasets: History and context.
\newblock {\em ACM Trans. Interact. Intell. Syst.}, 5(4), December 2015.

\bibitem{wang2023onlineclusteringbanditsmisspecified}
Zhiyong Wang, Jize Xie, Xutong Liu, Shuai Li, and John C.~S. Lui.
\newblock Online clustering of bandits with misspecified user models, 2023.

\bibitem{li2022privacy}
Tan Li and Linqi Song.
\newblock Privacy-preserving communication-efficient federated multi-armed bandits.
\newblock {\em IEEE Journal on Selected Areas in Communications}, 2022.

\bibitem{li2020federated}
Tian Li, Anit~Kumar Sahu, Manzil Zaheer, Maziar Sanjabi, Ameet Talwalkar, and Virginia Smith.
\newblock Federated optimization in heterogeneous networks.
\newblock {\em Proceedings of Machine learning and systems}, 2:429--450, 2020.

\bibitem{shi2021federatedwithpersonalization}
Chengshuai Shi, Cong Shen, and Jing Yang.
\newblock Federated multi-armed bandits with personalization.
\newblock In {\em Proc. {AISTATS}}, pages 2917--2925, 2021.

\bibitem{dubey2020differentially}
Abhimanyu Dubey and Alex Pentland.
\newblock Differentially-private federated linear bandits.
\newblock In {\em Proc. {NeurIPS}}, pages 6003--6014, 2020.

\bibitem{huang2021federated}
Ruiquan Huang, Weiqiang Wu, Jing Yang, and Cong Shen.
\newblock Federated linear contextual bandits.
\newblock In {\em Proc. {NeurIPS}}, 2021.

\bibitem{li2021asynchronous}
Chuanhao Li and Hongning Wang.
\newblock Asynchronous upper confidence bound algorithms for federated linear bandits.
\newblock In {\em Proc. {AISTATS}}, 2022.

\bibitem{dai2020federated}
Zhongxiang Dai, Bryan Kian~Hsiang Low, and Patrick Jaillet.
\newblock Federated bayesian optimization via thompson sampling.
\newblock {\em Advances in Neural Information Processing Systems}, 33:9687--9699, 2020.

\bibitem{dai2021differentially}
Zhongxiang Dai, Bryan Kian~Hsiang Low, and Patrick Jaillet.
\newblock Differentially private federated {Bayesian} optimization with distributed exploration.
\newblock In {\em Proc. {NeurIPS}}, 2021.

\bibitem{li2022communicationkcb}
Chuanhao Li, Huazheng Wang, Mengdi Wang, and Hongning Wang.
\newblock Communication efficient distributed learning for kernelized contextual bandits.
\newblock {\em Advances in Neural Information Processing Systems}, 35:19773--19785, 2022.

\bibitem{li2022communicationglb}
Chuanhao Li and Hongning Wang.
\newblock Communication efficient federated learning for generalized linear bandits.
\newblock {\em Advances in Neural Information Processing Systems}, 35:38411--38423, 2022.

\bibitem{zhu2021federated}
Zhaowei Zhu, Jingxuan Zhu, Ji~Liu, and Yang Liu.
\newblock Federated bandit: A gossiping approach.
\newblock {\em Proc. ACM Meas. Anal. Comput. Syst.}, 5(1):1--29, 2021.

\bibitem{ciucanu2022samba}
Radu Ciucanu, Pascal Lafourcade, Gael Marcadet, and Marta Soare.
\newblock {SAMBA}: A generic framework for secure federated multi-armed bandits.
\newblock {\em JAIR}, 73:737--765, 2022.

\bibitem{blaser2024federated}
Ethan Blaser, Chuanhao Li, and Hongning Wang.
\newblock Federated linear contextual bandits with heterogeneous clients.
\newblock In {\em International Conference on Artificial Intelligence and Statistics}, pages 631--639. PMLR, 2024.

\bibitem{fan2024federated}
Li~Fan, Ruida Zhou, Chao Tian, and Cong Shen.
\newblock Federated linear bandits with finite adversarial actions.
\newblock {\em Advances in Neural Information Processing Systems}, 36, 2024.

\bibitem{yang2024federated}
Hantao Yang, Xutong Liu, Zhiyong Wang, Hong Xie, John~CS Lui, Defu Lian, and Enhong Chen.
\newblock Federated contextual cascading bandits with asynchronous communication and heterogeneous users.
\newblock In {\em Proceedings of the AAAI Conference on Artificial Intelligence}, volume~38, pages 20596--20603, 2024.

\bibitem{solanki2024fairness}
Sambhav Solanki, Shweta Jain, and Sujit Gujar.
\newblock Fairness and privacy guarantees in federated contextual bandits.
\newblock {\em arXiv preprint arXiv:2402.03531}, 2024.

\bibitem{fourati2024federated}
Fares Fourati, Mohamed-Slim Alouini, and Vaneet Aggarwal.
\newblock Federated combinatorial multi-agent multi-armed bandits.
\newblock {\em arXiv preprint arXiv:2405.05950}, 2024.

\bibitem{wang2024towards}
Zibo Wang, Yifei Zhu, Dan Wang, and Zhu Han.
\newblock Towards fair and scalable trial assignment in federated bandits: A shapley value approach.
\newblock {\em IEEE Transactions on Big Data}, 2024.

\bibitem{li2024federated}
Wenjie Li, Qifan Song, Jean Honorio, and Guang Lin.
\newblock Federated x-armed bandit.
\newblock In {\em Proceedings of the AAAI Conference on Artificial Intelligence}, volume~38, pages 13628--13636, 2024.

\bibitem{wei2024incentivized}
Zhepei Wei, Chuanhao Li, Tianze Ren, Haifeng Xu, and Hongning Wang.
\newblock Incentivized truthful communication for federated bandits.
\newblock {\em arXiv preprint arXiv:2402.04485}, 2024.

\bibitem{li2024fedconpe}
Zhuohua Li, Maoli Liu, and John Lui.
\newblock Fedconpe: Efficient federated conversational bandits with heterogeneous clients.
\newblock {\em arXiv preprint arXiv:2405.02881}, 2024.

\bibitem{ICML09_yue2009interactively}
Yisong Yue and Thorsten Joachims.
\newblock Interactively optimizing information retrieval systems as a dueling bandits problem.
\newblock In {\em Proc. ICML}, pages 1201--1208, 2009.

\bibitem{ICML11_yue2011beat}
Yisong Yue and Thorsten Joachims.
\newblock Beat the mean bandit.
\newblock In {\em Proc. ICML}, pages 241--248, 2011.

\bibitem{WSDM14_zoghi2014relative}
Masrour Zoghi, Shimon~A Whiteson, Maarten De~Rijke, and Remi Munos.
\newblock Relative confidence sampling for efficient on-line ranker evaluation.
\newblock In {\em Proc. WSDM}, pages 73--82, 2014.

\bibitem{ICML14_ailon2014reducing}
Nir Ailon, Zohar Karnin, and Thorsten Joachims.
\newblock Reducing dueling bandits to cardinal bandits.
\newblock In {\em Proc. ICML}, pages 856--864, 2014.

\bibitem{ICML14_zoghi2014relative}
Masrour Zoghi, Shimon Whiteson, Remi Munos, and Maarten Rijke.
\newblock Relative upper confidence bound for the k-armed dueling bandit problem.
\newblock In {\em Proc. ICML}, pages 10--18, 2014.

\bibitem{COLT15_komiyama2015regret}
Junpei Komiyama, Junya Honda, Hisashi Kashima, and Hiroshi Nakagawa.
\newblock Regret lower bound and optimal algorithm in dueling bandit problem.
\newblock In {\em Proc. COLT}, pages 1141--1154, 2015.

\bibitem{ICML15_gajane2015relative}
Pratik Gajane, Tanguy Urvoy, and Fabrice Cl{\'e}rot.
\newblock A relative exponential weighing algorithm for adversarial utility-based dueling bandits.
\newblock In {\em Proc. ICML}, pages 218--227, 2015.

\bibitem{UAI18_saha2018battle}
Aadirupa Saha and Aditya Gopalan.
\newblock Battle of bandits.
\newblock In {\em Proc. UAI}, pages 805--814, 2018.

\bibitem{AISTATS19_saha2019active}
Aadirupa Saha and Aditya Gopalan.
\newblock Active ranking with subset-wise preferences.
\newblock In {\em Proc. AISTATS}, pages 3312--3321, 2019.

\bibitem{ALT19_saha2019pac}
Aadirupa Saha and Aditya Gopalan.
\newblock Pac battling bandits in the plackett-luce model.
\newblock In {\em Proc. ALT}, pages 700--737, 2019.

\bibitem{AISTATS22_saha2022exploiting}
Aadirupa Saha and Suprovat Ghoshal.
\newblock Exploiting correlation to achieve faster learning rates in low-rank preference bandits.
\newblock In {\em Proc. AISTATS}, pages 456--482, 2022.

\bibitem{ICML23_zhu2023principled}
Banghua Zhu, Michael Jordan, and Jiantao Jiao.
\newblock Principled reinforcement learning with human feedback from pairwise or k-wise comparisons.
\newblock In {\em Proc. ICML}, pages 43037--43067, 2023.

\bibitem{arXiv23_di2023variance}
Qiwei Di, Tao Jin, Yue Wu, Heyang Zhao, Farzad Farnoud, and Quanquan Gu.
\newblock Variance-aware regret bounds for stochastic contextual dueling bandits.
\newblock {\em arXiv:2310.00968}, 2023.

\bibitem{verma2024neural}
Arun Verma, Zhongxiang Dai, Xiaoqiang Lin, Patrick Jaillet, and Bryan Kian~Hsiang Low.
\newblock Neural dueling bandits.
\newblock {\em arXiv preprint arXiv:2407.17112}, 2024.

\bibitem{wu2021clustering}
Junda Wu, Canzhe Zhao, Tong Yu, Jingyang Li, and Shuai Li.
\newblock Clustering of conversational bandits for user preference learning and elicitation.
\newblock In {\em Proceedings of the 30th ACM International Conference on Information \& Knowledge Management}, pages 2129--2139, 2021.

\bibitem{zong2016cascading}
Shi Zong, Hao Ni, Kenny Sung, Nan~Rosemary Ke, Zheng Wen, and Branislav Kveton.
\newblock Cascading bandits for large-scale recommendation problems.
\newblock {\em arXiv preprint arXiv:1603.05359}, 2016.

\bibitem{hazan2023introductiononlineconvexoptimization}
Elad Hazan.
\newblock Introduction to online convex optimization, 2023.

\end{thebibliography}
\bibliographystyle{unsrt}







\appendix



\newpage

\newpage
\appendix
\onecolumn


\section{Details of the Vanilla {\tt FLDB-GD} Algorithm}
\label{app:sec:fldb-gd}
In this section, we provide the detailed descriptions of our vanilla \texttt{FLDB-GD} algorithm (Sec.~\ref{subsec:algo:gd}).

\begin{algorithm}[h]
	\begin{algorithmic}[1]
        \STATE {\textbf{Initialization:}} $W_{\text sync} = \frac{\lambda}{\kappa_{\mu}}I_{d \times d}, \theta_{\text sync} = 0_{d \times 1}$
		\FOR{$t= 1, \ldots, T$}
            \STATE Receive contexts $\mathcal{X}_{t,i}$
            \STATE Choose the first arm  $x_{t,1,i} = \arg\max_{x\in\mathcal{X}_{t,i}}\theta_{\text{sync}}^\top \phi(x)$
            \STATE Choose the second arm $x_{t,2,i} = \arg\max_{x\in\mathcal{X}_{t,i}} \theta_{\text{sync}}^\top \left( \phi(x) - \phi(x_{t,1,i}) \right) + \frac{\beta_t}{\kappa_\mu}\lVert{\phi(x) - \phi(x_{t,1,i})}\rVert_{W_{\text{sync}}^{-1}}$
            \STATE Observe the preference feedback $y_{t,i} = \mathbbm{1}(x_{t,1,i}\succ x_{t,2,i})$ 
            \STATE Calculate $W_{\text{new},i} = \left(\phi(x_{t,1,i}) - \phi(x_{t,2,i})\right) \left(\phi(x_{t,1,i}) - \phi(x_{t,2,i})\right)^{\top}$
            
		\FOR{$s= 0, \ldots, M-1$}
                \STATE Receive $\theta^{(s)}$ from Central Server
                \STATE Calculate local gradient $\nabla\mathcal{L}_t^{i}(\theta^{(s)})$ \eqref{eq:loss:func:fed:local}
                \STATE Send $\nabla\mathcal{L}_t^{i}(\theta^{(s)})$ back to Central Server
            \ENDFOR

            \STATE Send $W_{\text{new},i}$ to Central Server
            \STATE Receive $W_{\text{sync}}$ from Central Server
		\ENDFOR
	\end{algorithmic}
\caption{\texttt{FLDB-GD} (Agent $i$)}
\label{algo:new:gd:agent}
\end{algorithm}

\begin{algorithm}[h]
	\begin{algorithmic}[1]

            \IF{Receive $\{\nabla\mathcal{L}_t^{i}(\theta^{(s)})\}_{i=1,\ldots,N}$ from all $N$ agents}
                \STATE Aggregates local gradients: $\nabla\mathcal{L}_t^{\text{fed}}({\theta}^{(s)}) = \sum^N_{i=1} \nabla\mathcal{L}_t^{i}(\theta^{(s)}) + \lambda \theta^{(s)}$ \eqref{eq:loss:func:fed}
                \STATE Updates parameter: $\theta^{(s+1)} = \theta^{(s)} - \eta \nabla\mathcal{L}_t^{\text{fed}}({\theta}^{(s)})$
                \STATE Broadcast $\theta^{(s)}$ to all agents
            \ENDIF
            \IF{Receive $\{W_{\text{new},i}\}_{i=1,\ldots,N}$ from all $N$ agents}
                \STATE Calculate $W_{\text{sync}} \leftarrow W_{\text{sync}} + \sum^N_{i=1} W_{\text{new},i}$
                \STATE Broadcast $\{W_{\text{sync}}\}$ to all agents
            \ENDIF
	\end{algorithmic}
\caption{\texttt{FLDB-GD} (Central Server)}
\label{algo:new:gd:server}
\end{algorithm}


\textbf{Details of Algo.~\ref{algo:new:gd:agent} and Algo.~\ref{algo:new:gd:server}.}
In each round $t$, agent $i$ receives a set of context vectors $\mathcal{X}_{t,i}$ and uses the synchronized parameter $\theta_{\text{sync}}$ to select a pair of arms (lines 4–5 of Algo.~\ref{algo:new:gd:agent}). After observing the preference feedback $y_{t,i} = \mathbbm{1}(x_{t,1,i} \succ x_{t,2,i})$, the agent computes a matrix $W_{\text{new},i}$ based on the feature difference between the chosen arms (line 7 of Algo.~\ref{algo:new:gd:agent}).

The server and agents then perform $M$ steps of federated gradient descent. In each step $s = 0, \ldots, M-1$, the server broadcasts the current parameter $\theta^{(s)}$ to all agents (line 9 of Algo.~\ref{algo:new:gd:agent}). Each agent computes its local gradient $\nabla \mathcal{L}_t^i(\theta^{(s)})$ (line 10 of Algo.~\ref{algo:new:gd:agent}) and sends it to the server (line 11 of Algo.~\ref{algo:new:gd:agent}). Upon receiving all local gradients, the server aggregates them as
\[
\nabla \mathcal{L}_t^{\text{fed}}(\theta^{(s)}) = \sum_{i=1}^N \nabla \mathcal{L}_t^i(\theta^{(s)}) + \lambda \theta^{(s)},
\]
and updates the parameter by gradient descent: $\theta^{(s+1)} = \theta^{(s)} - \eta \nabla \mathcal{L}_t^{\text{fed}}(\theta^{(s)})$ (lines 2–3 of Algo.~\ref{algo:new:gd:server}). The updated parameter is then broadcast to all agents (line 4 of Algo.~\ref{algo:new:gd:server}).

After completing $M$ steps, each agent sends its $W_{\text{new},i}$ to the server (line 12 of Algo.~\ref{algo:new:gd:agent}). The server updates the global matrix as
\[
W_{\text{sync}} \leftarrow W_{\text{sync}} + \sum_{i=1}^N W_{\text{new},i}
\]
(line 6 of Algo.~\ref{algo:new:gd:server}), and broadcasts the updated $W_{\text{sync}}$ to all agents (line 7 of Algo.~\ref{algo:new:gd:server}), completing round $t$.

\section{More Experimental Details and Results}
\label{app:sec:more:exp:details}
In all our experiments, we set the horizon to $T = 500$ and use the same hyperparameter $\lambda = 1 / T$ to ensure fair comparisons.
For \texttt{FLDB-OGD}, we choose $\alpha=1000$.
Every curve represents the mean and standard error computed over three independent runs.

\subsection{Heterogeneous Rewards}
\label{sec:heterogeneous}
\begin{figure*}[h!]
\vspace{-1.2mm}
     \centering
     \begin{tabular}{ccc}
         \includegraphics[width=0.33\linewidth]{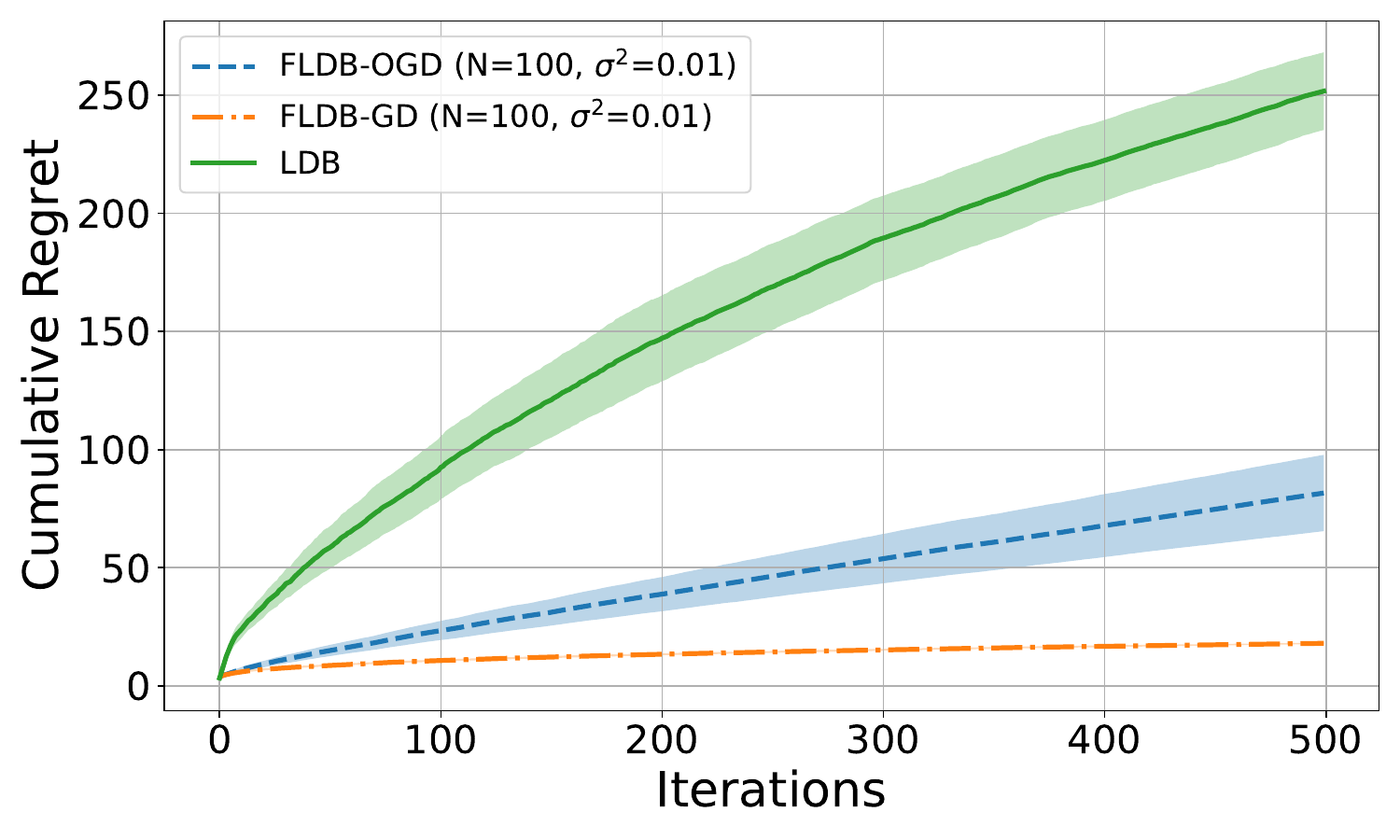} & \hspace{-5mm} 
         \includegraphics[width=0.33\linewidth]{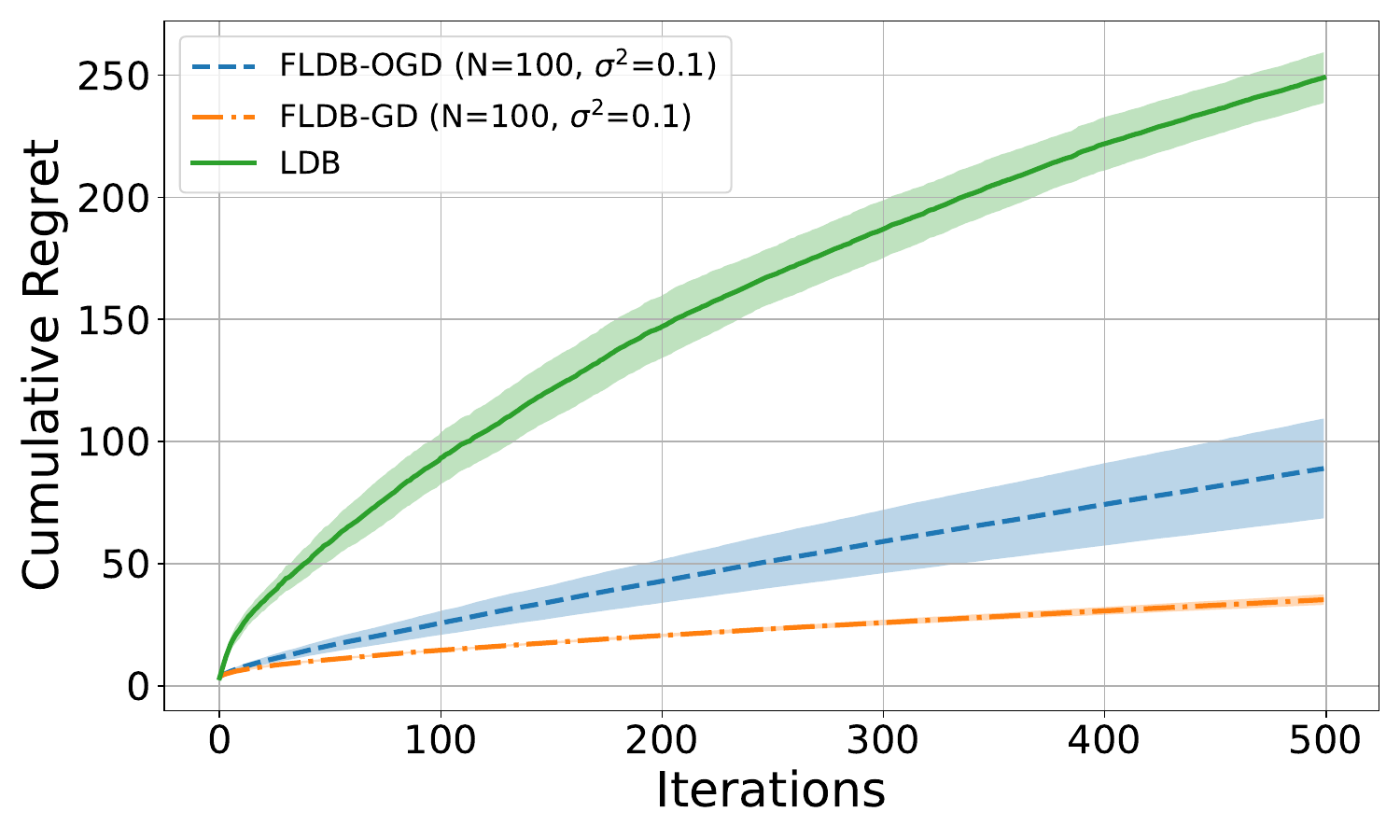}& \hspace{-3mm}
         \includegraphics[width=0.33\linewidth]{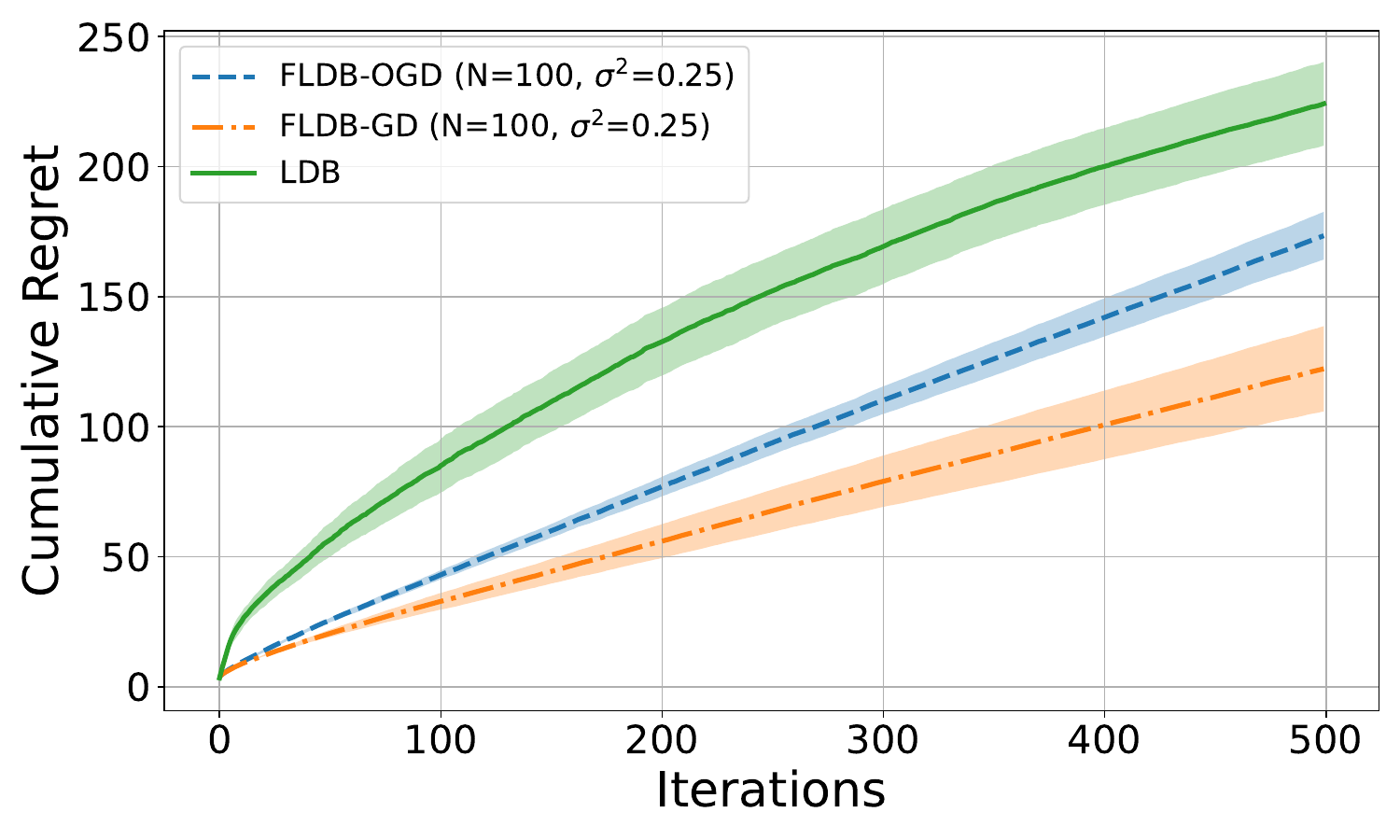} \\
         {\hspace{-2mm} (a) $\sigma^2=0.01$} & {\hspace{-5mm} (b) $\sigma^2=0.1$} & {\hspace{-6mm} (c) $\sigma^2=0.25$} 
     \end{tabular}
\vspace{-2.8mm}
     \caption{
     Cumulative regret with varying $\sigma^2$ under $K=5, d=5$.}
     \label{fig:exp:heterogeneous}
\vspace{-1.4mm}
\end{figure*}
The results for the experiments with heterogeneous rewards are shown in Fig.~\ref{fig:exp:heterogeneous}. Both \texttt{FLDB-GD} and \texttt{FLDB-OGD} exhibit robustness under different levels of agent-specific perturbations, maintaining competitive performance as the degree of heterogeneity increases.


\subsection{Communication trade-off}
\label{sec:comm:trade:off}

\begin{figure*}[h!]
\vspace{-1.2mm}
     \centering
     \begin{tabular}{ccc}
         \includegraphics[width=0.33\linewidth]{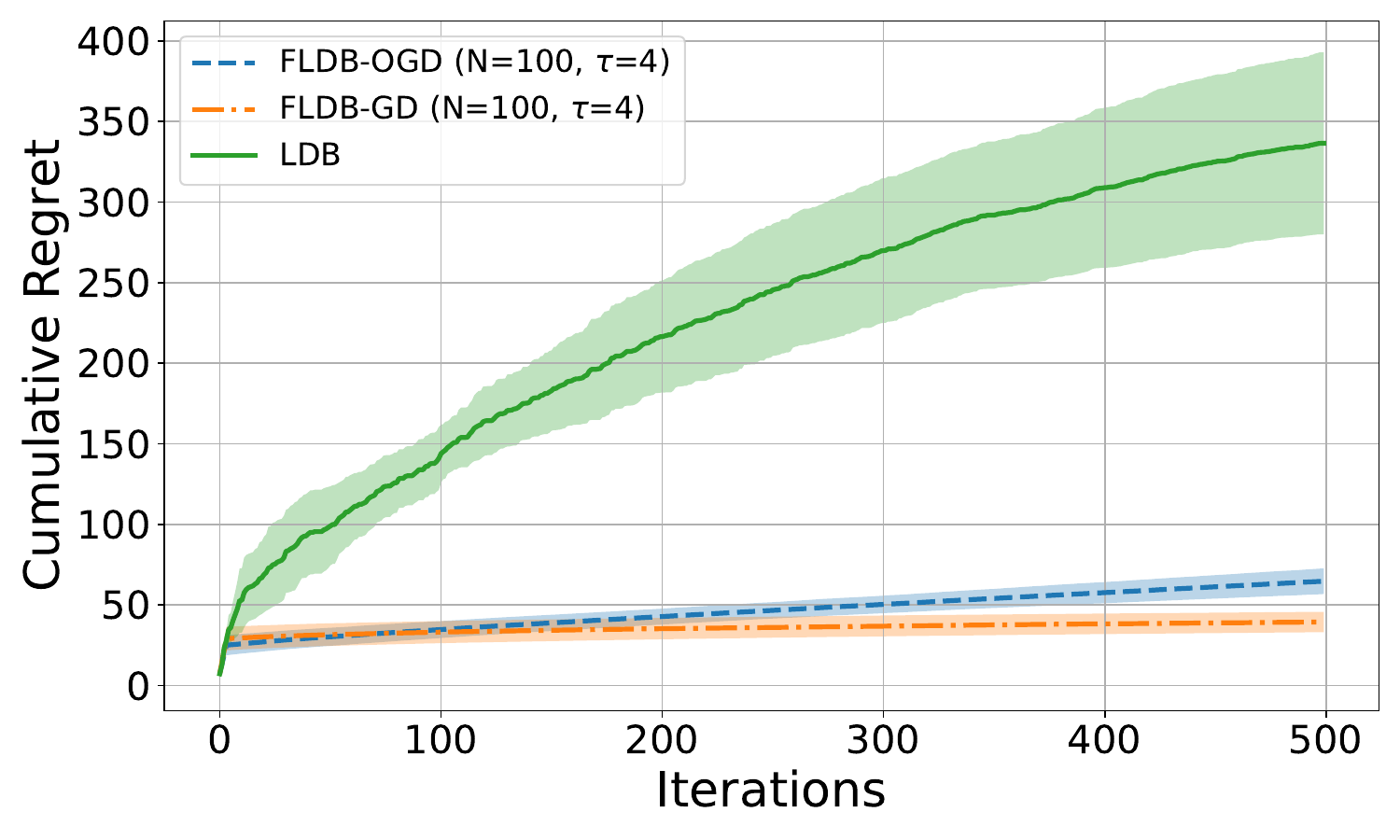} & \hspace{-5mm} 
         \includegraphics[width=0.33\linewidth]{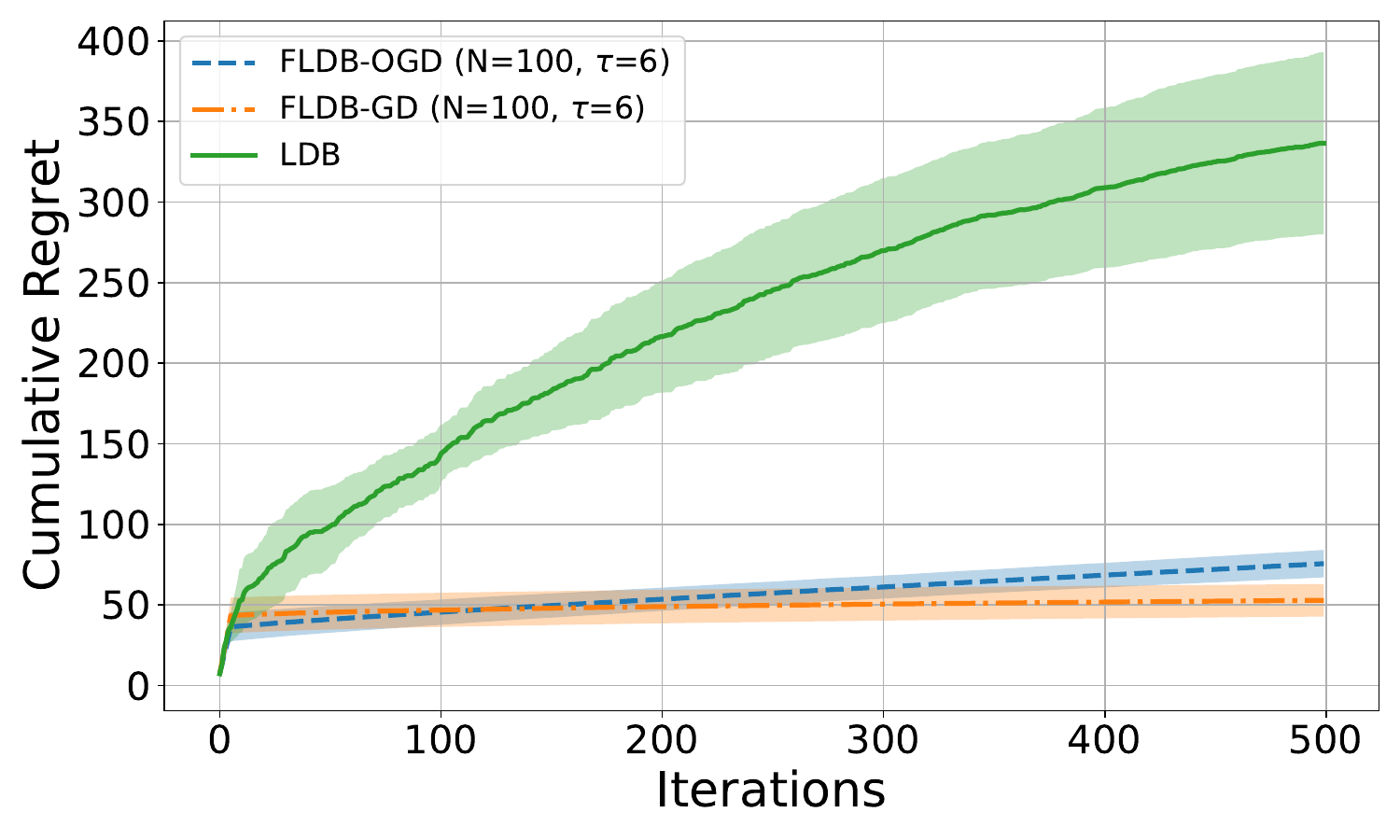}& \hspace{-3mm}
         \includegraphics[width=0.33\linewidth]{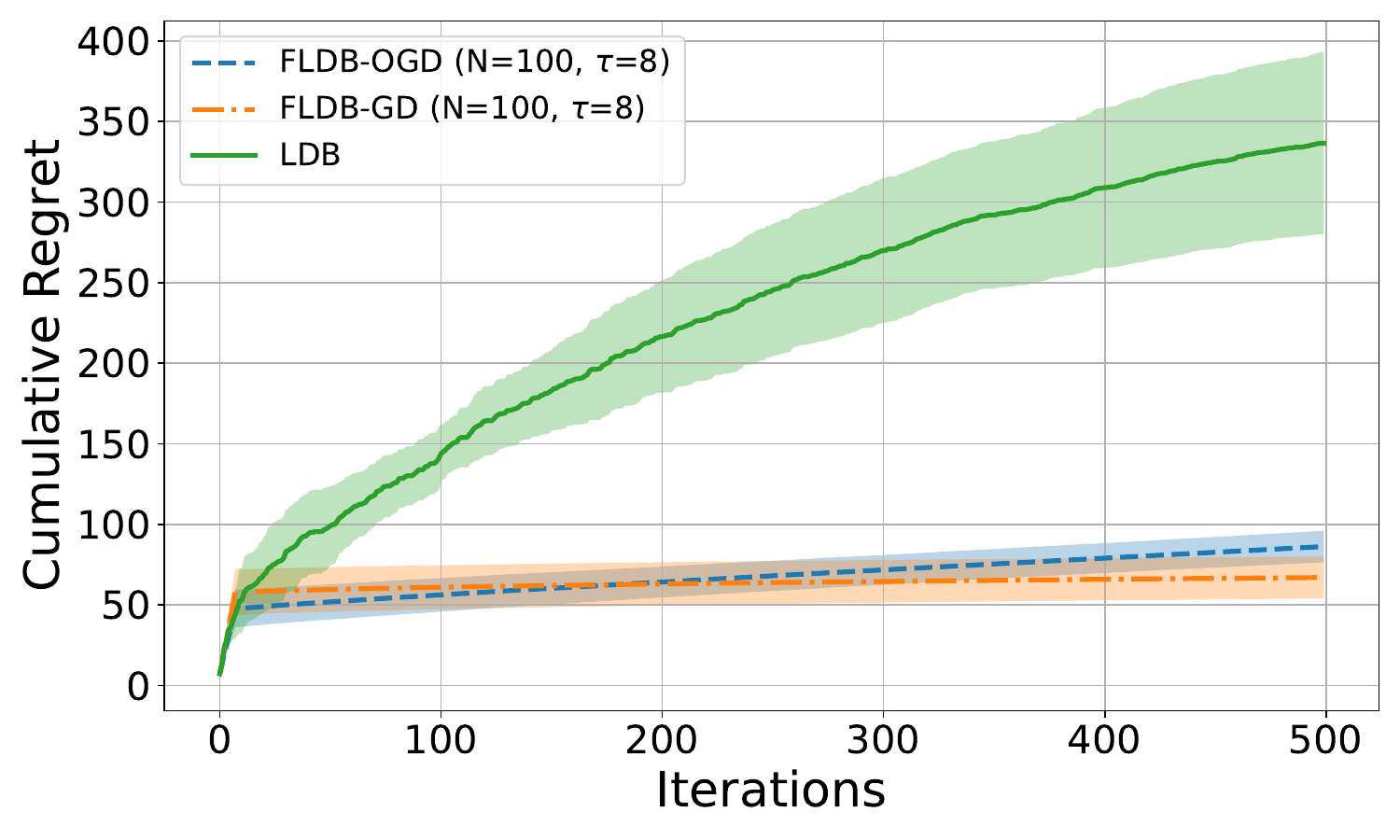} \\
         {\hspace{-2mm} (a) $\tau=4$} & {\hspace{-5mm} (b) $\tau=6$} & {\hspace{-6mm} (c) $\tau=8$} 
     \end{tabular}
\vspace{-2.8mm}
     \caption{
     Cumulative regret with varying $\tau$ under $N=100, K=10, d=5$.}
     \label{fig:exp:comm:round_1}
\vspace{-1.4mm}
\end{figure*}


\begin{figure*}[h!]
\vspace{-1.2mm}
     \centering
     \begin{tabular}{cc}
         \includegraphics[width=0.5\linewidth]{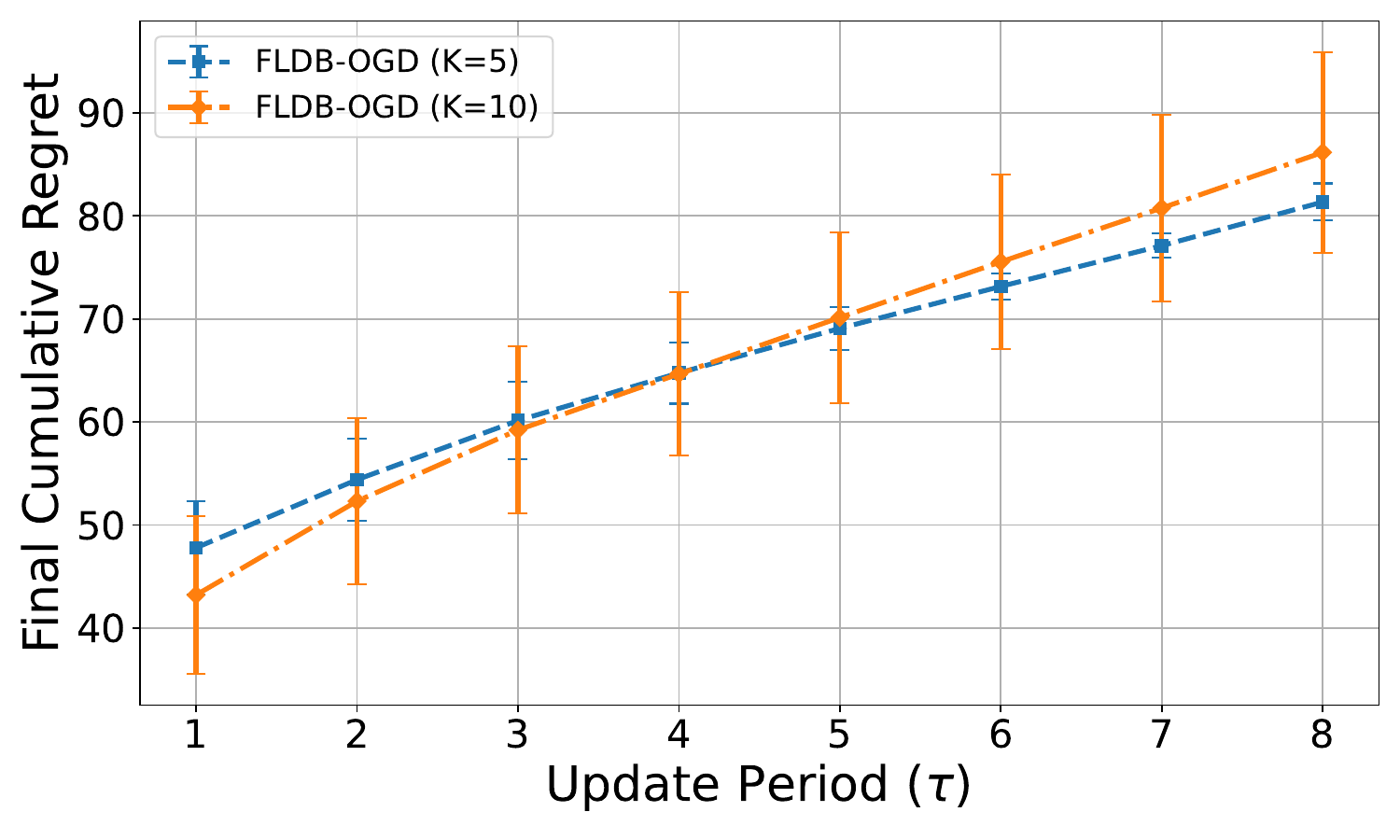} & \hspace{-5mm} 
         \includegraphics[width=0.5\linewidth]{ICML_2025/figures/final_regret_vs_communication_rounds_N100_K_multiple.pdf} \\
         {\hspace{-2mm} (a) $N=100$} & {\hspace{-5mm} (b) $N=100$} 
     \end{tabular}
\vspace{-2.8mm}
    \caption{
    Final cumulative regret versus local update period $\tau$ for different numbers of arms $K$ with $N=100$.
    }
     \label{fig:exp:comm:round_2}
\vspace{-1.4mm}
\end{figure*}
We conduct experiments with varying local update periods $\tau = 1, 2, \dots, 8$. As shown in Figure~\ref{fig:exp:comm:round_1}, \texttt{FLDB-OGD} consistently outperforms the \texttt{LDB} baseline across all values of $\tau$. In addition, we examine the effect of $\tau$ on the final cumulative regret, as depicted in Figure~\ref{fig:exp:comm:round_2}a. The results show that the final regret increases with larger $\tau$, which aligns with our theoretical analysis in Proposition~\ref{prop:regret:local:update}.

When we plot the final cumulative regret against the number of communication rounds (i.e., $T/\tau$ for a fixed horizon $T$), as shown in Figure~\ref{fig:exp:comm:round_2}b, the trade-off between communication frequency and regret becomes more evident.

\subsection{Real-world Experiment}
\label{sec:realworld}

For the experiment on the MovieLens dataset~\cite{movielens}, we follow the setting introduced in~\cite{wang2023onlineclusteringbanditsmisspecified}. We extract 200 users who rate most and 200 items with most ratings. 
Then we construct a binary
matrix $\boldsymbol{H}^{200\times200}$ based on the user rating \cite{wu2021clustering, zong2016cascading}: if the user rating is greater than 3, the feedback is 1; otherwise, the feedback is 0.
 Then we select the first 20 rows $\boldsymbol{H}^{20\times200}_{1}$ to generate the feature vectors by SVD. The remaining 180 rows $\boldsymbol{F}^{180\times200}$ is used as the feedback matrix, meaning user $i$ receives $\boldsymbol{F}(i,j)$ as feedback while choosing item $j$. At time $t$, we first randomly select one user as environment and sample $5$ items for the algorithms to choose from. 
 
\subsection{Computational Resources and Runtime}
\label{sec:resource}
Each trial was executed on a single NVIDIA A800 GPU. For the setting with $K=10, d=5$, the average runtime per trial is as follows: 
\begin{itemize}
    \item \texttt{LDB}: 264.51 s 
    \item \texttt{FLDB-OGD}$(N=100, \tau=1)$: 70.25 s
    \item \texttt{FLDB-OGD}$(N=100, \tau=2)$: 68.34 s
    \item \texttt{FLDB-GD}$(N=100)$: $>$ 9 h but $<$ 10 h
\end{itemize}

\section{Proof of Lemma~\ref{lemma:concentration:theta}}
\label{appendix:proof:theta}
We first consider the classic linear dueling bandit setting, namely the single agent scenario (N=1). In this scenario, we ignore the subscript $i$ (we use $\{x_{t,1}, x_{t,2}, y_t, r_t\}$ instead of $\{x_{t,1,i}, x_{t,2,i}, y_{t,i}, r_{t,i}\}$).
\begin{lemma}[Single Agent]
\label{lemma:concentration:theta:single}
Let $\beta_t \triangleq \sqrt{2\log(1/\delta) + d\log\left( 1 + t\kappa_{\mu}/(d\lambda) \right)}$

With probability of at least $1-\delta$, we have that for all $t=1,\ldots,T$
\[
\norm{\theta - \theta_{t}}_{V_t} \leq \frac{\beta_t}{\kappa_\mu}.
\]
\end{lemma}
\begin{proof}
Let $\phi: \mathbb{R}^{d'} \rightarrow \mathbb{R}^{d}$ be a feature map such that $f(x) = \theta^\top\phi(x)$ and $d \geq d^\prime$. 
In the case of linear bandits, $\phi(x) = x$ and $d=d'$.
In iteration $s$, define $\widetilde{\phi}_s = \phi(x_{s,1}) - \phi(x_{s,2})$.
Define $\widetilde{f}_{s} = f(x_{s,1}) - f(x_{s,2}) =\theta_f^{\top} \widetilde{\phi}_s$.

For any $\theta_{f'} \in\mathbb{R}^{d}$, define 
\[
G_t(\theta_{f^\prime}) = \sum_{s=1}^t \left(\mu(\theta_{f^\prime}^\top\widetilde{\phi}_s ) - \mu(\theta^\top\widetilde{\phi}_s) \right) \widetilde{\phi}_s  + \lambda \theta_{f'}.
\]
For $\lambda'\in(0, 1)$, setting $\theta_{\bar{f}} = \lambda' \theta_{f^\prime_1} + (1 - \lambda')\theta_{f^\prime_2}$.
and using mean-value theorem, we get:
\begin{align}
    \label{eqn:glb}
    G_t(\theta_{f^\prime_1}) - G_t(\theta_{f^\prime_2}) &= \left[\sum_{s=1}^t \dot\mu(\theta_{\bar{f}}^\top\widetilde{\phi}_s)\widetilde{\phi}_s \widetilde{\phi}_s^\top + \lambda \mathbf{I} \right](\theta_{f^\prime_1} - \theta_{f^\prime_2}) & \left( \theta_f \text{ is constant} \right)  \nonumber \\
\end{align}
Define $M_t = \left[\sum_{s=1}^t \dot\mu(\theta_{\bar{f}}^\top\widetilde{\phi}_s)\widetilde{\phi}_s \widetilde{\phi}_s^\top + \lambda \mathbf{I} \right]$, and define $V_t = \sum_{s=1}^t \widetilde{\phi}_s \widetilde{\phi}_s^\top + \frac{\lambda}{\kappa_\mu} \mathbf{I}$.
Then we have that $M_t \geq \kappa_\mu V_t$ and that $V^{-1}_t \geq \kappa_\mu M^{-1}_t$.

\begin{align*}
    \norm{G_t(\theta_{t})}_{V_t^{-1}}^2 &=  \norm{G_t(\theta) - G_t(\theta_{t})}_{V_t^{-1}}^2 = \norm{M_t (\theta - \theta_t)}_{V_t^{-1}}^2 & \left( G_t(\theta) = 0 \text{ by definition} \right)\\
    & = (\theta - \theta_t)^{\top} M_t V_t^{-1} M_t (\theta - \theta_t)\\
    &\geq (\theta - \theta_t)^{\top} M_t \kappa_\mu M_t^{-1} M_t (\theta - \theta_t)\\
    & = \kappa_\mu(\theta - \theta_t)^{\top} M_t (\theta - \theta_t)\\
    & \geq \kappa_\mu(\theta - \theta_t)^{\top} \kappa_\mu V_t (\theta - \theta_t)\\
    & = \kappa_\mu^2 (\theta - \theta_t)^{\top} V_t (\theta - \theta_t)\\
    & = \kappa_\mu^2 \norm{\theta - \theta_{t}}^2_{V_t}  & \left(\text{as } ||x||_A^2 = x^\top A x \right)
\end{align*}
The first inequality is because $V^{-1}_t \geq \kappa_\mu M^{-1}_t$, and the second inequality follows from $M_t \geq \kappa_\mu V_t$.

\noindent
Let $f_{t,s} = \theta_t^{\top} \widetilde{\phi}_{s}$, in which $\theta_t$ is our empirical estimate of the parameter $\theta$.
Now using \ref{eqn:glb}, we have that
\begin{align*}
    \norm{G_t(\theta_{t})}_{V_t^{-1}}^2 &=  \norm{G_t(\theta) - G_t(\theta_{t})}_{V_t^{-1}}^2 & \left( G_t(\theta) = 0 \text{ by definition} \right) \\
    & \ge (\kappa_\mu V_t(\theta - \theta_{t}))^\top V_t^{-1} \kappa_\mu V_t(\theta - \theta_{t}) & \left(\text{as } ||x||_A^2 = x^\top A x \right) \\
    & = \kappa_\mu^2 (\theta - \theta_{t})^\top V_t V_t^{-1}  V_t(\theta - \theta_{t}) & \left(\text{as } V_t^\top = V_t \text{ and } \kappa_\mu \text{ is constant} \right) \\
    & = \kappa_\mu^2 \norm{\theta - \theta_{t}}^2_{V_t}  & \left(\text{as } ||x||_A^2 = x^\top A x \right)
\end{align*}

This allows us to show that 
\begin{align*}
    \norm{\theta - \theta_{t}}^2_{V_t} &\le \frac{1}{\kappa_\mu^2} \norm{G_t(\theta_{t})}_{V_t^{-1}}^2 \\
    &= \frac{1}{\kappa_\mu^2} \norm{\sum_{s=1}^t (\mu(\theta_{t}^\top \widetilde{\phi}_s ) - \mu(\theta^\top \widetilde{\phi}_s) ) \widetilde{\phi}_s + \lambda \theta_t}_{V_t^{-1}}^2 & \left(\text{by definition of } G_t(\theta_{t}) \right) \\
    &= \frac{1}{\kappa_\mu^2} \norm{\sum_{s=1}^t (\mu(f_{t,s}) - \mu(\widetilde{f}_s) ) \widetilde{\phi}_s + \lambda \theta_t}_{V_t^{-1}}^2 & \left(\text{see definitions of } f_{t,s}\, \text{and } \widetilde{f}_s \right) \\
    &= \frac{1}{\kappa_\mu^2} \norm{\sum_{s=1}^t (\mu(f_{t,s}) - (y_s - \epsilon_s) ) \widetilde{\phi}_s + \lambda \theta_t}_{V_t^{-1}}^2 & \left(\text{as } y_s = \mu(\widetilde{f}_s) + \epsilon_s \right) \\
    &= \frac{1}{\kappa_\mu^2} \norm{\sum_{s=1}^t \left(\mu(f_{t,s}) - y_s\right) \widetilde{\phi}_s + \sum_{s=1}^t\epsilon_s \widetilde{\phi}_s  + \lambda \theta_t}_{V_t^{-1}}^2 \\
    &\leq \frac{1}{\kappa_\mu^2} \norm{\sum_{s=1}^t\epsilon_s \widetilde{\phi}_s}_{V_t^{-1}}^2.
\end{align*}
The last step follows from the fact that $\theta_{t}$ is computed using MLE by solving the following equation:
\begin{equation}
    \sum_{s=1}^t \left(\mu\left( \theta_t^{\top} \widetilde{\phi}_s \right) - y_s\right) \widetilde{\phi}_s + \lambda \theta_t = 0,
\end{equation}
which is ensured by Assumption~\ref{assumption:zero:gradient}.

With this, we have
\begin{equation}
    \label{eqn:parUB}
    \norm{\theta - \theta_{t}}_{V_t}^2 \le \frac{1}{\kappa_\mu^2} \norm{\sum_{s=1}^t\epsilon_s \widetilde{\phi}_s}_{V_t^{-1}}^2.
\end{equation}


Denote $V_t \triangleq \sum_{s=1}^t \widetilde{\phi}_s \widetilde{\phi}_s^\top + \frac{\lambda}{\kappa_\mu} \mathbf{I}$ and $V \triangleq \frac{\lambda}{\kappa_\mu} \mathbf{I}$.
Denote the observation noise $\epsilon_s = y_s - \mu(f(x_{s,1}) - f(x_{s,2}))$. Note that the sequence of observation noises $\{\epsilon_s\}$ is $1$-sub-Gaussian.

Next, we can apply Theorem 1 from \cite{NIPS11_abbasi2011improved}, to obtain
\begin{equation}
\norm{\sum_{s=1}^t\epsilon_s \widetilde{\phi}_s}_{V_t^{-1}}^2 \leq 2\log\left( \frac{\det(V_t)^{1/2}}{\delta\det(V)^{1/2}} \right),
\end{equation}
which holds with probability of at least $1-\delta$.

Next, based on our assumption that $\norm{\widetilde{\phi}_s}_2 \leq 1$ (Assumption \ref{assumption:basic}), according to Lemma 10 from \cite{NIPS11_abbasi2011improved}, we have that
\begin{equation}\label{eqn:single}
\det(V_t) \leq \left( \lambda/\kappa_\mu +t/ d \right)^d.
\end{equation}
Therefore, 
\begin{equation}
\sqrt{\frac{\det{V_t}}{\det(V)}} \leq \sqrt{\frac{\left( \lambda/\kappa_\mu +t L^2 / d \right)^d}{(\lambda/\kappa_\mu)^d}} = \left( 1 + t\kappa_{\mu}/(d\lambda) \right)^{\frac{d}{2}}
\label{eq:upper:bound:det:Vt:V}
\end{equation}
This gives us
\begin{equation}
\norm{\sum_{s=1}^t\epsilon_s \widetilde{\phi}_s}_{V_t^{-1}}^2 \leq 2\log\left( \frac{\det(V_t)^{1/2}}{\delta\det(V)^{1/2}} \right) \leq 2\log(1/\delta) + d\log\left( 1 + t\kappa_{\mu}/(d\lambda) \right)
\label{eq:upper:bound:proof:theta:interm}
\end{equation}

Combining \eqref{eqn:parUB} and \eqref{eq:upper:bound:proof:theta:interm}, we have that
\begin{equation}
\norm{\theta - \theta_{t}}_{V_t}^2 \le \frac{1}{\kappa_\mu^2} \left(2\log(1/\delta) + d\log\left( 1 + t\kappa_{\mu}/(d\lambda) \right)\right) = \frac{\beta_t^2}{\kappa_\mu^2},
\end{equation}
which completes the proof.

\end{proof}

To establish the multi-agent scenario, it suffices to replace~\eqref{eqn:single} with  

\begin{equation}\label{eqn:multi}  
\det(V_t) \leq \left( \lambda/\kappa_\mu + tN / d \right)^d, 
\end{equation}  

while keeping the rest unchanged.

\section{Proof of Theorem~\ref{theorem:gd}}\label{appendix:theorem:gd}
To prove Theorem~\ref{theorem:gd}, we first need the following lemma:
\begin{lemma}
\label{lemma:concentration:square:std}
\[
\sum^T_{t=1}\norm{\phi(x_{t,1}) - \phi(x_{t,2})}_{V_{t-1}^{-1}}^2 \leq 2 d\log \left( 1 + TN\kappa_{\mu}/(d\lambda) \right).
\]
\end{lemma}

Similar to the proof of Lemma~\ref{lemma:concentration:theta} in Appendix~\ref{appendix:proof:theta}, we first prove the single agent case:
\begin{lemma}[Single Agent]
\label{lemma:concentration:square:std:single}
\[
\sum^T_{t=1}\norm{\phi(x_{t,1}) - \phi(x_{t,2})}_{V_{t-1}^{-1}}^2 \leq 2 d\log \left( 1 + T\kappa_{\mu}/(d\lambda) \right).
\]
\end{lemma}
\begin{proof}
We denote $\widetilde{\phi}_t = \phi(x_{t,1}) - \phi(x_{t,2})$.
Recall that we have assumed that $\norm{\phi(x_{t,1}) - \phi(x_{t,2})}_2 \leq 1$.
It is easy to verify that $V_{t-1} \succeq \frac{\lambda}{\kappa_\mu} I$ and hence $V_{t-1}^{-1} \preceq \frac{\kappa_\mu}{\lambda}I$.
Therefore, we have that $\norm{\widetilde{\phi}_t}_{V_{t-1}^{-1}}^2 \leq \frac{\kappa_\mu}{\lambda} \norm{\widetilde{\phi}_t}_{2}^2 \leq  \frac{\kappa_\mu}{\lambda}$. We choose $\lambda$ such that $\frac{\kappa_\mu}{\lambda} \leq 1$, which ensures that $\norm{\widetilde{\phi}_t}_{V_{t-1}^{-1}}^2 \leq 1$.
Our proof here mostly follows from Lemma 11 of \cite{NIPS11_abbasi2011improved}. To begin with, note that $x\leq 2\log(1+x)$ for $x\in[0,1]$. Then we have that 
\begin{equation}
\begin{split}
\sum^T_{t=1}\norm{\widetilde{\phi}_t}_{V_{t-1}^{-1}}^2 &\leq \sum^T_{t=1} 2\log\left(1 + \norm{\widetilde{\phi}_t}_{V_{t-1}^{-1}}^2\right)\\
&= 2 \left( \log\det V_T - \log\det V \right)\\
&= 2 \log \frac{\det V_T}{\det V}\\
&\leq 2\log \left( \left( 1 + T \kappa_{\mu}/(d\lambda) \right)^{d} \right)\\
&= 2 d\log \left( 1 + T\kappa_{\mu}/(d\lambda) \right).
\end{split}
\end{equation}
The second inequality follows from \eqref{eq:upper:bound:det:Vt:V}.
This completes the proof.
\end{proof}

Lemma~\ref{lemma:concentration:square:std:single} can be applied to~\eqref{eqn:multi} to extend the conclusion to the \( N \)-agent case, resulting in Lemma~\ref{lemma:concentration:square:std}.

Then we can derive the regret bound of single agent case:
\begin{lemma}
\label{lemma:ucb:diff:single}
In any iteration $t=1,\ldots,T$, for all $x,x'\in\mathcal{X}_t$, with probability of at least $1-\delta$, we have that
\[
|\left(f(x) - f(x')\right) - \theta_t^{\top}\left( \phi(x) - \phi(x') \right)| \leq \frac{\beta_t}{\kappa_\mu}\norm{\phi(x) - \phi(x')}_{V_{t-1}^{-1}}.
\]
\end{lemma}
\begin{proof}
\begin{equation}
\begin{split}
|\left(f(x) - f(x')\right) - \theta_t^{\top}\left( \phi(x) - \phi(x') \right)| &= |\theta^{\top} \left[(\phi(x) - \phi(x')\right] - \theta_t^{\top}\left[ \phi(x) - \phi(x') \right]|\\
&= | \left(\theta - \theta_t\right)^{\top} \left[\phi(x) - \phi(x') \right]|\\
&\leq \norm{\theta - \theta_t}_{V_{t-1}} \norm{\phi(x) - \phi(x')}_{V_{t-1}^{-1}}\\
&\leq \frac{\beta_t}{\kappa_\mu}\norm{\phi(x) - \phi(x')}_{V_{t-1}^{-1}},
\end{split}
\end{equation}
in which the last inequality follows from Lemma \ref{lemma:concentration:theta:single}.
\end{proof}

\begin{theorem}[Single Agent]
Let $\beta_t \triangleq \sqrt{2\log(1/\delta) + d\log\left( 1 + t\kappa_{\mu}/(d\lambda) \right)}$.
With probability of at least $1-\delta$, we have that
\[
R_T \leq \frac{3}{2} \frac{1}{\kappa_\mu} \sqrt{2\log(1/\delta) + d\log\left( 1 + T\kappa_{\mu}/(d\lambda) \right)} \sqrt{T 2 d\log \left( 1 + T\kappa_{\mu}/(d\lambda) \right)}.
\]
\end{theorem}
\begin{proof}
\begin{equation}
\begin{split}
2 r_t &= f(x^*_t) - f(x_{t,1}) + f(x^*_t) - f(x_{t,2})\\
&\stackrel{(a)}{\leq} \theta_t^\top \left( \phi(x^*_t) - \phi(x_{t,1}) \right) + \frac{\beta_t}{\kappa_\mu}\norm{\phi(x^*_t) - \phi(x_{t,1})}_{V_{t-1}^{-1}} +  \theta_t^\top \left( \phi(x^*_t) - \phi(x_{t,2}) \right) + \frac{\beta_t}{\kappa_\mu}\norm{\phi(x^*_t) - \phi(x_{t,2})}_{V_{t-1}^{-1}}\\
&= \theta_t^\top \left( \phi(x^*_t) - \phi(x_{t,1}) \right) + \frac{\beta_t}{\kappa_\mu}\norm{\phi(x^*_t) - \phi(x_{t,1})}_{V_{t-1}^{-1}} + \\
&\qquad \theta_t^\top \left( \phi(x^*_t) - \phi(x_{t,1}) \right) + \theta_t^\top \left( \phi(x_{t,1}) - \phi(x_{t,2}) \right) + \frac{\beta_t}{\kappa_\mu}\norm{\phi(x^*_t) - \phi(x_{t,1}) + \phi(x_{t,1}) - \phi(x_{t,2})}_{V_{t-1}^{-1}}\\
&\stackrel{(b)}{\leq} 2 \theta_t^\top \left( \phi(x^*) - \phi(x_{t,1}) \right) + 2 \frac{\beta_t}{\kappa_\mu}\norm{\phi(x^*) - \phi(x_{t,1})}_{V_{t-1}^{-1}} + \\
&\qquad \theta_t^\top \left( \phi(x_{t,1}) - \phi(x_{t,2}) \right) + \frac{\beta_t}{\kappa_\mu}\norm{\phi(x_{t,1}) - \phi(x_{t,2})}_{V_{t-1}^{-1}}\\
&\stackrel{(c)}{\leq} 2 \theta_t^\top \left( \phi(x_{t,2}) - \phi(x_{t,1}) \right) + 2 \frac{\beta_t}{\kappa_\mu}\norm{\phi(x_{t,2}) - \phi(x_{t,1})}_{V_{t-1}^{-1}} + \\
&\qquad \theta_t^\top \left( \phi(x_{t,1}) - \phi(x_{t,2}) \right) + \frac{\beta_t}{\kappa_\mu}\norm{\phi(x_{t,1}) - \phi(x_{t,2})}_{V_{t-1}^{-1}}\\
&\leq \theta_t^\top \left( \phi(x_{t,2}) - \phi(x_{t,1}) \right) + 3 \frac{\beta_t}{\kappa_\mu}\norm{\phi(x_{t,2}) - \phi(x_{t,1})}_{V_{t-1}^{-1}} \\
&\stackrel{(d)}{\leq} 3 \frac{\beta_t}{\kappa_\mu}\norm{\phi(x_{t,1}) - \phi(x_{t,2})}_{V_{t-1}^{-1}} \\
\end{split}
\label{eq:upper:bound:inst:regret:single}
\end{equation}
Step $(a)$ follows from Lemma \ref{lemma:ucb:diff:single}. Step $(b)$ makes use of the triangle inequality.
Step $(c)$ follows from the way in which we choose the second arm $x_{t,2}$: $x_{t,2} = \arg\max_{x\in\mathcal{X}_t} \theta_t^\top \left( \phi(x) - \phi(x_{t,1}) \right) + \frac{\beta_t}{\kappa_\mu}\norm{\phi(x) - \phi(x_{t,1})}_{V_{t-1}^{-1}}$.
Step $(d)$ results from the way in which we select the first arm: $x_{t,1} = \arg\max_{x\in\mathcal{X}_t}\theta_t^\top \phi(x)$.

Recall that $\beta_T = \sqrt{2\log(1/\delta) + d\log\left( 1 + T \kappa_{\mu}/(d\lambda) \right)}$.

\begin{equation}
\begin{split}
R_T = \sum^T_{t=1} r_t &\leq \sum^T_{t=1} \frac{3}{2} \frac{\beta_t}{\kappa_\mu}\norm{\phi(x_{t,1}) - \phi(x_{t,2})}_{V_{t-1}^{-1}}\\
&\leq \frac{3}{2} \frac{\beta_T}{\kappa_\mu} \sqrt{T \sum^T_{t=1}\norm{\phi(x_{t,1}) - \phi(x_{t,2})}_{V_{t-1}^{-1}}^2}\\
&\leq \frac{3}{2} \frac{\beta_T}{\kappa_\mu} \sqrt{T 2 d\log \left( 1 + T \kappa_{\mu}/(d\lambda) \right)}.
\end{split}
\end{equation}
\end{proof}

Now we begin to prove Theorem~\ref{theorem:gd}.

We use $V_t$ to denote the covariance matrix after iteration $t$: $V_t \triangleq \sum^t_{\tau=1}\sum^N_{i=1} \widetilde{\phi}_{\tau,i} \widetilde{\phi}_{\tau,i}^{\top} + \frac{\lambda}{\mu} \mathbf{I}$, in which $\widetilde{\phi}_{t,i} = \phi(x_{t,1,i}) - \phi(x_{t,2,i})$.
Consider a hypothetical agent which chooses all $T \times N$ pairs of arms in a round-robin fashion, i.e., it chooses $\{(x_{1,1,1},x_{1,2,1}), (x_{1,1,2},x_{t,2,2}), \ldots, (x_{1,1,N},x_{t,2,N}), \ldots, (x_{T,1,N},x_{T,2,N})\}$.
Define the covariance matrix for this hypothetical agent as $\widetilde{V}_{t,i} \triangleq V_{t} + \sum^i_{j=1} \widetilde{\phi}_{t,j} \widetilde{\phi}_{t,j}^{\top}$.
That is, the covariance matrix $\widetilde{V}_{t,i}$ consists of the information from the previous $t$ iterations, as well as the information from first $i$ agents in iteration $t+1$.

Define the set of "bad iterations" as $\mathcal{C} = \{ t=1,\ldots,T | \frac{\det V_{t}}{ \det V_{t-1}} >2 \}$, and we use $|\mathcal{C}|$ to represent the size of the set $\mathcal{C}$. Denoting $V_0=V=\frac{\lambda}{\kappa_\mu} \mathbf{I}$, we can show that 
\begin{equation}
\frac{\det V_T}{\det V} = \prod_{t=1,\ldots,T} \frac{\det V_t}{\det V_{t-1}} \geq \prod_{t \in \mathcal{C}} \frac{\det V_t}{\det V_{t-1}} > 2^{|\mathcal{C}|}.
\end{equation}
This implies that $|\mathcal{C}| \leq \log \frac{\det V_T}{\det V}$.

Next, we analyze the instantaneous regret in a good iteration, i.e., an iteration $t$ for which $\frac{\det V_{t}}{ \det V_{t-1}} \leq 2$.
Let $\beta_t \triangleq \sqrt{2\log(1/\delta) + d\log\left( 1 + t N\kappa_{\mu}/(d\lambda) \right)}$.
{ Note that we still have $\norm{\theta - \theta_{t}}_{V_{t-1}} \leq \frac{\beta_t}{\kappa_\mu}$, where $\theta_t$ refers to $\theta_{sync}$ at time t}.
Following the analysis of \ref{eq:upper:bound:inst:regret:single}, we can show that with probability of at least $1-\delta$,
\begin{equation}
\begin{split}
r_{t,i} &= f(x_{t,i}^*) - f(x_{t,1,i}) + f(x_{t,i}^*) - f(x_{t,2,i})\\
&\leq 3 \frac{\beta_t}{\kappa_\mu}\norm{\phi(x_{t,1,i}) - \phi(x_{t,2,i})}_{V_{t-1}^{-1}}.
\end{split}
\end{equation}

Next, making use of Lemma 12 of \cite{NIPS11_abbasi2011improved}, we can show that
\begin{equation}
\begin{split}
\norm{\phi(x_{t,1,i}) - \phi(x_{t,2,i})}_{V_{t-1}^{-1}} &\leq \norm{\phi(x_{t,1,i}) - \phi(x_{t,2,i})}_{\widetilde{V}_{t-1}^{-1}} \sqrt{ \frac{\det \widetilde{V}_{t-1,i}}{\det V_{t-1}}}\\
&\leq \norm{\phi(x_{t,1,i}) - \phi(x_{t,2,i})}_{\widetilde{V}_{t-1}^{-1}} \sqrt{ \frac{\det V_{t}}{\det V_{t-1}}}\\
&\leq \sqrt{2} \norm{\phi(x_{t,1,i}) - \phi(x_{t,2,i})}_{\widetilde{V}_{t-1}^{-1}}.
\end{split}
\end{equation}
This allows us to show that
\begin{equation}
r_{t,i} \leq 3 \sqrt{2} \frac{\beta_t}{\kappa_\mu} \norm{\phi(x_{t,1,i}) - \phi(x_{t,2,i})}_{\widetilde{V}_{t-1}^{-1}}
\end{equation}

Now the overall regrets of all agents in all iterations can be analyzed as:
\begin{equation}
\begin{split}
R_{T,N} &= \sum^T_{t=1}\sum^N_{i=1} r_{t,i} = |\mathcal{C}| N 2 + \sum_{t\notin\mathcal{C}} \sum^N_{i=1} r_{t,i} \\
&\leq 2N \log \frac{\det V_T}{\det V} + \sum_{t\notin\mathcal{C}} \sum^N_{i=1} r_{t,i}\\
&\leq 2N \log \frac{\det V_T}{\det V} + \sum^T_{t=1} \sum^N_{i=1} r_{t,i}\\
&\leq 2N \log \frac{\det V_T}{\det V} + 3 \sqrt{2} \frac{\beta_T}{\kappa_\mu} \sum_{t=1}^T \sum^N_{i=1} \norm{\phi(x_{t,1,i}) - \phi(x_{t,2,i})}_{\widetilde{V}_{t-1}^{-1}} \\
&\leq 2N d\log \left( 1 + T N \kappa_{\mu}/(d\lambda) \right) + 3 \sqrt{2} \frac{\beta_T}{\kappa_\mu} \sqrt{T 2 d\log \left( 1 + T N \kappa_{\mu}/(d\lambda) \right)} \\
\end{split}
\end{equation}
Ignoring all log factors, we have that
\begin{equation}
R_{T,N} = \widetilde{O}\left(N d + \frac{\sqrt{d}}{\kappa_\mu} \sqrt{Td}\right) = \widetilde{O}\left(N d + \frac{d}{\kappa_\mu} \sqrt{T}\right)
\end{equation}

\section{Proof of Lemma \ref{lemma:ogd:convergence} }
\label{app:sec:proof:lemma:ogd:convergence}
Lemma~\ref{lemma:ogd:init} ensures all the $\theta_t=\arg\max_{\theta'} $(optimal value of the~\ref{eq:loss:func:fed} at iteration t)  lie in a ball centered with the ground truth $\theta$.
\begin{lemma}\label{lemma:ogd:init}
    For horizon T, let $r \triangleq \sqrt{ \frac{2\log(1/\delta) + d\log\left( 1 + T N \kappa_{\mu}/(d\lambda) \right)}{\lambda \kappa_{\mu}}}$. With probability at least $1- \delta$, we have $\|\theta_t - \theta \| \leq r$ for $t= 1 \dots  T$
\end{lemma}

\begin{proof}
    Since  $\norm{\theta - \theta_{t}}_{V_{t-1}} \leq \frac{\beta_t}{\kappa_\mu}$ with probability at least $1-\delta$, we have 
    \begin{equation}
           \frac{\beta_T}{\kappa_\mu} \geq \norm{\theta - \theta_{t}}_{V_{t-1}}  \geq  \sqrt{\lambda_{min}({V_{t-1}})} \norm{\theta - \theta_{t}} \geq \sqrt{\frac{\lambda}{\kappa_{\mu}}}\norm{\theta - \theta_{t}}
    \end{equation}
    Thus, 
    \begin{equation}
    \norm{\theta - \theta_{t}} \leq \frac{\beta_T}{\sqrt{\kappa_\mu \lambda}} =  \sqrt{ \frac{2\log(1/\delta) + d\log\left( 1 + T N \kappa_{\mu}/(d\lambda) \right)}{\lambda \kappa_{\mu}}}
    \end{equation}   
\end{proof}
We will also use Proposition 1 in \cite{li2017provablyoptimalalgorithmsgeneralized}:
\begin{proposition}[Proposition 1 in \cite{li2017provablyoptimalalgorithmsgeneralized}]
\label{prop:eigenvalue}
Define $V_{n+1} = \sum_{t=1}^n X_tX_t^T$, where $X_t$ is drawn IID from some distribution in unit ball $\mathbbm{B}^d$. Furthermore, let $\Sigma := E[X_t X_t^T]$ be the second moment matrix, let $B, \delta>0$ be two positive constants. Then there exists positive, universal constants $C_1$ and $C_2$ such that $\lambda_{\min} (V_{n+1}) \geq B$ with probability at least $1-\delta$, as long as
\begin{equation*}
    n \geq \left( \frac{C_1 \sqrt{d} + C_2 \sqrt{\log(1/\delta)}}{\lambda_{\min} (\Sigma)} \right)^2 + \frac{2B}{\lambda_{\min} (\Sigma)}.
\end{equation*}
\end{proposition}

To prove Lemma~\ref{lemma:ogd:convergence}, we also need to show Propsition~\ref{prop:alpha:convex}.


\begin{proof}[Proof of Proposition~\ref{prop:alpha:convex}]

    Recall  
    \begin{equation*}
    f_s^{\text{fed}}(\theta') =
    \begin{cases} 
    \displaystyle \sum_{i=1}^{N} l_s^{i}(\theta'), & \quad \text{if } s \neq 1, \\[15pt]
    \displaystyle  \sum_{i=1}^{N} l_1^{i}(\theta') + \frac{\lambda}{2} \| \theta' \|_2^2, & \quad \text{if } s = 1.
    \end{cases}
    \label{eq:loss:func:fed:local:sum:ogd:2}
    \end{equation*}
    we have 
    \begin{align}\label{eq:hessian}
    \nabla^2 l^i_{s} (\theta^{\prime}) &= \mu^{\prime} {(\theta^{\prime}}^T \widetilde{\phi}_{s,i})\widetilde{\phi}_{s,i}  \widetilde{\phi}_{s,i}^T \nonumber, \\
    \nabla^2 f_{s}^{\text{fed}} (\theta^{\prime}) &= \sum_{i=1}^{N} 
 \mu^{\prime} {(\theta^{\prime}}^T \widetilde{\phi}_{s,i}) \widetilde{\phi}_{s,i}  \widetilde{\phi}_{s,i}^T + \lambda \mathbbm{1}(s=1) \succeq \kappa_{\mu} \sum_{i=1}^{N} \widetilde{\phi}_{s,i}  \widetilde{\phi}_{s,i}^T ~~\text{for}~\forall~\theta^{\prime}\in \mathbbm{B}_{3r}. 
    \end{align}
    The last inequality comes from Assumption~\ref{assumption:basic} and $\lambda > 0$.
    Since we update $\widetilde\theta_t$ every $N$ rounds (for clients), for the next $N$ rounds (for client), the pulled arms are only dependent on $\widetilde\theta_t$. 
Therefore, the feature vectors of pulled arms among the next $N$ rounds are IID. Thus, $\widetilde{\phi}_i$ is drawn IID in unit ball $\mathbbm{B}^d$. By applying Proposition~\ref{prop:eigenvalue} letting $B=\frac{\alpha}{\kappa_{\mu}}$ and with Assumption~\ref{assumption:fed}, we have with probability at least $1-\delta$,
$$
\nabla^2 f_s\left(\theta^{\prime}\right) \succeq \alpha I_d \text { if } N \geq\left(\frac{C_1 \sqrt{d}+C_2 \sqrt{\log \left(1/\delta\right)}}{\lambda_f}\right)^2 +\frac{2 \alpha}{\kappa_{\mu} \lambda_f}
$$

Namely, $f_s^{\text{fed}}\left(\theta^{\prime}\right)$ is $\alpha$-stronly convex for $\forall~ 1 \leqslant s \leqslant t$.
 \end{proof}

Then we will prove Lemma~\ref{lemma:ogd:convergence}.

\begin{proof}
    Now we have:
    \begin{align*}
        \theta_t &=\arg\min_{\theta'}\mathcal{L}_t^{\text{fed}}(\theta') \\
        \widehat{\theta}^{(t+1)} &= \pi_S \left( \widehat{\theta}^{(t)} - \eta_{t} \nabla f_t^{\text{fed}}({ \widehat{\theta}}^{(t)}) \right) ~~\text{for}~t\geq1 \\
        \widetilde{\theta}^{(t)} &= \frac{1}{t} \sum_{j = 1} ^{t} \widehat{\theta}^{(j)} 
    \end{align*}
    From Lemma~\ref{lemma:ogd:init} and Algorithm~\ref{algo:new:ogd:server} we have $\theta_t, \widehat{\theta}^{(t+1)}, \widetilde{\theta}^{(t)} \in \mathbbm{B}_{3r}$ with probability at least $1-\frac{\delta}{2}$. Then by Jensen's inequality, we have 
    \begin{equation}\label{eq:jensen}
        \sum_{s=1}^{t}f_s^{\text{fed}}(\widehat{\theta}^{(s)}) \ge \sum_{s=1}^{t}f_s^{\text{fed}}(\widetilde{\theta}^{(t)}).
    \end{equation}
    Since $\mathcal{S}$ is convex, we apply Theorem 3.3 of Section 3.3.1 in \cite{hazan2023introductiononlineconvexoptimization} and get
    \begin{equation*}
             \sum_{s=1}^t \left( f_s^{\text{fed}}(\widehat{\theta}^{(s)}) - f_s^{\text{fed}}(\theta_t) \right) \leq \frac{G^2}{2\alpha} (1+\log t) ~~\forall ~t \ge1
    \end{equation*}

where $G$ satisfies $G^2 \geq E\|\nabla f_{s}^{\text{fed}}\|^2$. Additionally, we have 
\begin{equation}
\|\nabla f_{s}^{\text{fed}}(\theta')\| =  \| \sum_{i=1}^N \left(\mu\left( {\theta'}^{\top} \widetilde{\phi}_{s,i}\right) - y_{s,i}\right) \widetilde{\phi}_{s,i} \| \le N ,~\text{for} s > 1.
\label{eq:G:upper:bound}
\end{equation}
Thus, we have 
    \begin{equation*}
              \frac{N^2}{2\alpha} (1+\log t) \ge \sum_{s=1}^t \left( f_s^{\text{fed}}(\widehat{\theta}^{(s)}) - f_s^{\text{fed}}(\theta_t) \right) \ge\sum_{s=1}^{t} \left(f_s^{\text{fed}}(\widetilde{\theta}^{(t)}) - f_s^{\text{fed}}(\theta_t) \right) 
    \end{equation*}
By Proposition~\ref{prop:alpha:convex}, we know $f_s^{\text{fed}}(\theta^{\prime})$ is $\alpha$-strongly convex for $\forall~ 1 \leq s \leq t$ with probability at least $1-\frac{\delta}{2}$. Meanwhile, $\theta_t, \widehat{\theta}^{(t+1)}, \widetilde{\theta}^{(t)} \in \mathbbm{B}_{3r}$ with probability at least $1-\frac{\delta}{2}$. Then using the property of $\alpha$-strongly convex we have
\begin{equation}
    f_s^{\text{fed}}(\widetilde{\theta}^{(t)}) \ge f_s^{\text{fed}}(\theta_t) + \nabla f_{s}^{\text{fed}}(\theta_t)^{\top}\left(\widetilde{\theta}^{(t)} - \theta_t\right) + \frac{\alpha}{2}\|\widetilde{\theta}^{(t)} - \theta_t\|^2.
\end{equation}
Summing over $s$, we have that
\begin{equation}\label{eq:convex:ineq}
    \sum_{s=1}^{t}f_s^{\text{fed}}(\widetilde{\theta}_t) \ge \sum_{s=1}^{t}f_s^{\text{fed}}(\theta_t) + \left(\sum_{s=1}^{t}\nabla f_{s}^{\text{fed}}(\theta_t)^{\top}\right)\left(\widetilde{\theta}_t - \theta_t\right) + \frac{\alpha t}{2}\|\widetilde{\theta}_t - \theta_t\|^2.
\end{equation}
Since $\sum_{s=1}^{t}\nabla f_{s}(\theta_t) = 0$, we have that
\begin{equation}
    \frac{\alpha t}{2}\|\widetilde{\theta}^{(t)} - \theta_t\|^2 \le \sum_{s=1}^{t}f_s^{\text{fed}}(\widetilde{\theta}^{(t)}) - \sum_{s=1}^{t}f_s^{\text{fed}}(\theta_t) 
    = \sum_{s=1}^{t} \left(f_s^{\text{fed}}(\widetilde{\theta}^{(t)}) - f_s^{\text{fed}}(\theta_t)\right) 
    \le \frac{(N)^2}{2\alpha} (1+\log t). \nonumber
\end{equation}
This allows us to show that
\begin{equation}
    \|\widetilde{\theta}^{(t)} - \theta_t\|
    \le \frac{N}{\alpha} \sqrt{\frac{1+\log t}{t}}.
\end{equation}
From $\|\widetilde{\phi}_{s,i}\| \le 1$ we get $\lambda_{max} (V_{t}) \le t N + \frac{\lambda}{\kappa_{\mu}}$. Therefore,
\begin{align}
    \norm{\widetilde{\theta}^{(t)} - \theta_t}_{V_t} &\le
    \sqrt{\lambda_{max} (V_{t})}\|\widetilde{\theta}_t - \theta_t\| \\
    & \le \frac{\sqrt{t N + \frac{\lambda}{\kappa_{\mu}}}N}{\alpha} \sqrt{\frac{1+\log t}{t}}\\
    &\le \frac{N\sqrt{N + \frac{\lambda}{\kappa_{\mu}}}}{\alpha} \sqrt{1+\log t}.
\end{align}
The last inequality holds since $t\ge1$.

\end{proof}

\section{Proof of Theorem \ref{theorem:ogd} }

\begin{proof}
Following the definition in Section~\ref{sec:background}, we will analyze the instantaneous regret in a good iteration. To begin with, we have that
\begin{equation*}
    \norm{\theta - \widetilde{\theta}^{(t)}}_{V_t} \le  \norm{\theta - \theta_{t}}_{V_t} + \norm{\theta_t - \widetilde{\theta}^{(t)}}_{V_t}\le \frac{\beta_t}{\kappa_{\mu}} + \norm{\theta_t - \widetilde{\theta}^{(t)}}_{V_t} \le \frac{\beta_t}{\kappa
    _{\mu}}+\frac{NL\sqrt{N L^2 + \frac{\lambda}{\kappa_{\mu}}}}{\alpha} \sqrt{1+\log t},
\end{equation*}
and
\begin{equation}
\begin{split}
|\left(f(x) - f(x')\right) - {\left(\widetilde{\theta}^{(t)}\right)}^{\top}\left( \phi(x) - \phi(x') \right)| &= |\theta^{\top} \left[(\phi(x) - \phi(x')\right] - {\left(\widetilde{\theta}^{(t)}\right)}^{\top}\left[ \phi(x) - \phi(x') \right]|\\
&= | \left(\theta - \widetilde{\theta}^{(t)}\right)^{\top} \left[\phi(x) - \phi(x') \right]|\\
&\leq \norm{\theta - \widetilde{\theta}^{(t)}}_{V_{t-1}} \norm{\phi(x) - \phi(x')}_{V_{t-1}^{-1}}\\
&\leq \left( \frac{\beta_t}{\kappa
    _{\mu}}+\frac{N\sqrt{N + \frac{\lambda}{\kappa_{\mu}}}}{\alpha} \sqrt{1+\log t}\right)\norm{\phi(x) - \phi(x')}_{V_{t-1}^{-1}}. \nonumber
\end{split}
\end{equation}
Thus, for the instantaneous regret in a good iteration, with probability at least $1-\delta$, we have that
\begin{equation}
\begin{split}
r_{t,i} &= f(x_{t,i}^*) - f(x_{t,1,i}) + f(x_{t,i}^*) - f(x_{t,2,i})\\
&\leq 3 \left( \frac{\beta_t}{\kappa
    _{\mu}}+\frac{N\sqrt{N + \frac{\lambda}{\kappa_{\mu}}}}{\alpha} \sqrt{1+\log t}\right)\norm{\phi(x_{t,1,i}) - \phi(x_{t,2,i})}_{V_{t-1}^{-1}}.
\end{split}
\end{equation}
By following the proof in Appendix~\ref{appendix:theorem:gd},  the overall regrets of all agents in all iterations can be analyzed as:
\begin{equation}
\begin{split}
R_{T,N}
&\leq 2N d\log \left( 1 + T N \kappa_{\mu}/(d\lambda) \right) + \\
&\qquad 3 \sqrt{2} \left( \frac{\beta_T}{\kappa_{\mu}}+\frac{N\sqrt{N  + \frac{\lambda}{\kappa_{\mu}}}}{\alpha} \sqrt{1+\log T}\right)\sqrt{T 2 d\log \left( 1 + T N \kappa_{\mu}/(d\lambda) \right)} \\
\end{split}
\end{equation}
Ignoring all log factors, the final regret upper bound is given by:
    \begin{equation}
    R_{T,N} = \widetilde{O}\left(N d + (\frac{\sqrt{d}}{\kappa_\mu} + \frac{N^{\frac{3}{2}}}{\alpha})\sqrt{Td}\right) = \widetilde{O}\left(N d + (\frac{d}{\kappa_\mu} +  \frac{N^{\frac{3}{2}}\sqrt{d}}{\alpha})\sqrt{T}\right).
    \end{equation}
\end{proof}

\section{Proof of Proposition~\ref{prop:regret:local:update}}
\begin{proof}
When the number of local updates $\tau$ is larger than $1$, between two consecutive communication rounds, every agent collects $\tau>1$ observations. This gives rise to a total of $N\times \tau$ observations collected between two consecutive communication rounds.
Given that the arm feature vectors (i.e., contexts) in different iterations are sampled independently, 
this process can be equivalently viewed as having $N\times \tau$ agents with a local update period of $1$.
As a result, for a given horizon $T$, this is equivalent to running our \texttt{FLDB-OGD} algorithm with $T' = \frac{T}{\tau}$ iterations, $N'=N\times \tau$ agents, and  a local update period of $\tau'=1$.
This allows us to modify the result of Theorem \ref{theorem:ogd} to derive the following upper bound on the total cumulative regret:
    \begin{equation}
    \begin{split}
    R^{\tau}_{T,N} = R_{T',N'} & = \widetilde{O}\big(N' d + \frac{d}{\kappa_\mu} \sqrt{T'} +  \frac{N'^{\frac{3}{2}} \sqrt{d}}{\alpha} \sqrt{T'}\big) \\
    & = \widetilde{O}\big(\tau N d + \frac{d}{\sqrt{\tau} \kappa_\mu} \sqrt{T} +  \tau\frac{N^{\frac{3}{2}} \sqrt{d}}{\alpha} \sqrt{T}\big).
    \end{split}
    \end{equation}
\end{proof}

\newpage

\end{document}